\pdfoutput=1

\documentclass{article}

\usepackage{microtype}
\usepackage{graphicx}
\usepackage{booktabs} 

\usepackage{hyperref}

\usepackage[accepted]{icml2021}

\usepackage{amsmath,amsthm,amsfonts,amssymb}
\usepackage{caption}
\usepackage{subcaption}
\usepackage{mathtools}
\usepackage{font} 

\newtheorem{theorem}{Theorem}[section]
\newtheorem{lemma}[theorem]{Lemma}

\newtheorem{remark}[theorem]{Remark}

\newtheorem{proposition}[theorem]{Proposition}
\theoremstyle{definition}
\newtheorem{definition}[theorem]{Definition}

\newtheorem{assumption}{Assumption}

\def\##1\#{\begin{align}#1\end{align}}
\def\$#1\${\begin{align*}#1\end{align*}}

\let\hat\widehat
\let\tilde\widetilde

\def\given{{\,|\,}}
\def\biggiven{{\,\big|\,}}

\newcommand{\la}{\langle}
\newcommand{\ra}{\rangle}

\DeclareMathOperator*{\argmax}{arg\,max}
\DeclareMathOperator*{\argmin}{arg\,min}

\newcommand\mydef{\stackrel{\text{def}}{=}}

\makeatletter
\let\iftwo\if@twocolumn
\makeatother

\begin{document}

\twocolumn[
\icmltitle{Randomized Exploration for Reinforcement Learning with General Value Function Approximation}



\icmlsetsymbol{equal}{*}

\begin{icmlauthorlist}
\icmlauthor{Haque Ishfaq}{equal,mila,mcgill}
\icmlauthor{Qiwen Cui}{equal,peking}
\icmlauthor{Viet Nguyen}{mila,mcgill}
\icmlauthor{Alex Ayoub}{amii}
\icmlauthor{Zhuoran Yang}{princeton}
\icmlauthor{Zhaoran Wang}{northwestern}
\icmlauthor{Doina Precup}{mila,mcgill,deepmind}
\icmlauthor{Lin F.~Yang}{ucla}
\end{icmlauthorlist}

\icmlaffiliation{mila}{Mila}
\icmlaffiliation{peking}{School of Mathematical Science, Peking University}
\icmlaffiliation{deepmind}{DeepMind, Montreal}
\icmlaffiliation{mcgill}{School of Computer Science, McGill University}
\icmlaffiliation{amii}{Amii and Department of Computing Science, University
of Alberta}
\icmlaffiliation{northwestern}{ Industrial Engineering \& Management Sciences, Northwestern University}
\icmlaffiliation{princeton}{Department of Operations Research and Financial Engineering, Princeton University}
\icmlaffiliation{ucla}{Department of Electrical and Computer Engineering, University of California, Los Angeles}

\icmlcorrespondingauthor{Haque Ishfaq}{haque.ishfaq@mail.mcgill.ca}

\icmlkeywords{Machine Learning, ICML}

\vskip 0.3in
]



\printAffiliationsAndNotice{\icmlEqualContribution} 

\begin{abstract}
We propose a model-free reinforcement learning algorithm  inspired by the popular randomized least squares value iteration (RLSVI) algorithm as well as the optimism principle. Unlike existing upper-confidence-bound (UCB) based approaches, which are often computationally intractable, our algorithm  drives exploration by simply perturbing the training data with judiciously chosen i.i.d. scalar noises. To attain optimistic value function estimation without resorting to a UCB-style bonus, we introduce an optimistic reward sampling procedure. When the value functions can be represented by a function class $\mathcal{F}$, our algorithm achieves a worst-case regret bound of $\tilde{O}(\mathrm{poly}(d_EH)\sqrt{T})$ where $T$ is the time elapsed, $H$ is the planning horizon and $d_E$  is the \emph{eluder dimension} of $\mathcal{F}$. In the linear setting, our algorithm reduces to LSVI-PHE, a variant of RLSVI, that enjoys an $\tilde{\mathcal{O}}(\sqrt{d^3H^3T})$ regret. We complement the theory with an empirical evaluation across known difficult exploration tasks.
\end{abstract}

\section{Introduction}
The exploration-exploitation trade-off is a core problem in reinforcement learning (RL): an agent may need to sacrifice short-term rewards to achieve better long-term returns. A good RL algorithm should explore efficiently and find a near-optimal policy as quickly and robustly as possible. A big open problem is the design of provably efficient exploration when general function approximation is used to estimate the value function, i.e., the expectation of long-term return. In this work, we propose an exploration strategy inspired by the popular Randomized Least Squares Value Iteration (RLSVI) algorithm \citep{osband2016generalization,russo2019worst,zanette2019frequentist} as well as by the optimism principle \citep{Brafman01r-max-,jaksch2010near,jin2018q, jin2019provably,wang2020reinforcement}, which is efficient in both statistical and computational sense, and can be easily plugged into common RL algorithms, including UCB-VI \citep{azar2017minimax}, UCB-Q \citep{jin2018q} and OPPO \citep{cai2019provably}.

The main exploration idea is the well-known ``optimism in the face of uncertainty (OFU)'' principle, which leads to numerous upper confidence bound (UCB)-type algorithms. These algorithms compute statistical confidence regions for the model or the value function, given the observed history,  and perform the greedy policy with respect to these  regions, or upper confidence bounds. However, it is costly or even intractable to compute the upper confidence bound explicitly, especially for structured MDPs or general function approximations. For instance, in \citet{wang2020reinforcement}, computing the confidence bonus requires sophisticated sensitivity sampling and a width function oracle. The computational cost hinders the practical application of these UCB-type algorithms.

Another recently rediscovered exploration idea is Thompson sampling (TS) \citep{thompson1933likelihood, osband2013more}. It is motivated by the Bayesian perspective on RL, in which we have a prior distribution over the model or the value function; then we draw a sample from this distribution and compute a policy based on this sample. Theoretical guarantees exist for both Bayesian regret \citep{russo2013eluder} and worst-case regret \citep{agrawal2017optimistic} for this approach. Although TS is conceptually simple, in many cases the posterior is intractable to compute and the prior may not exist at all. Recently, approximate TS, also known as randomized least squares value iteration (RLSVI) or following the perturbed leader \cite{kveton2019randomized}, has received significant attention due to its good empirical performance. It has been proven that RLSVI enjoys sublinear worst-case or frequentist regret in tabular RL, by simply adding Gaussian noise on the reward \citep{russo2019worst,agrawal2020improved} However, in the improved bound for tabular MDP \citep{agrawal2020improved} and linear MDP \citep{zanette2019frequentist},  the uncertainty of the estimates still needs to be computed in order to perform optimistic exploration; it is unknown whether this can be removed. Moreover, this computation is difficult to do in the general function approximation setting.

In this work, we propose a novel exploration idea called optimistic reward sampling, which combines OFU and TS organically. The algorithm is surprisingly simple: we perturb the reward several times and act greedily with respect to the maximum of the estimated state-action values. The intuition is that after the perturbation, the estimate has a constant probability of being optimistic, and sampling multiple times guarantees that the maximum of these sampled estimates is optimistic with high probability. Thus, our algorithm utilizes approximate TS to achieve optimism.

Similar algorithms have been shown to work empirically, including SUNRISE \citep{lee2020sunrise}, NoisyNet \citep{fortunato2017noisy} and bootstrapped DQN \citep{osband2016deep}. However, the theoretical analysis of perturbation-based exploration is still missing. We prove that it enjoys near optimal regret $\Tilde{O}(\sqrt{H^3d^3T})$ for linear MDP and the sampling time is only $M=\Tilde{O}(d)$. We also prove similar bounds for the general function approximation case, by using the notion of eluder dimension \citep{russo2013eluder,wang2020reinforcement}. In addition, this algorithm is computationally efficient, as we no longer need to compute the upper confidence bound. In the experiments, we find that a small sampling time $M$ is sufficient to achieve good performance, which suggests that the theoretical choice of $M=\Tilde{O}(d)$ is too conservative in practice.

Optimistic reward sampling can be directly plugged into most RL algorithms, improving the sample complexity without harming the computational cost. The algorithm only needs to perform perturbed regression. To our best knowledge, this is the first online RL algorithm that is both computationally and statistically efficient with linear function approximation and general function approximation.
We hope optimistic reward sampling can be a large step towards bridging the gap between algorithms with strong theoretical guarantees and those with good computational performance.

\section{Preliminaries}\label{section:background}

We begin by introducing some necessary notations. For any positive integer $n$, we denote the set $\{1, 2, \ldots,n\}$ by $[n]$. For any set $A$,  $\la \cdot, \cdot \ra_A$ denotes the inner product over set $A$. For a positive definite matrix $A \in \mathbb{R}^{d \times d}$ and a vector $x \in \mathbb{R}^d$, we denote the norm of $x$ with respect to matrix $A$ by $\|x\|_A = \sqrt{x^TAx}$. We denote the cumulative distribution function of the standard Gaussian by $\Phi(\cdot)$. For function growth, we use $\Tilde{\mathcal{O}}(\cdot)$, ignoring poly-logarithmic factors. 

We consider episodic MDPs of the form $(\mathcal{S}, \mathcal{A}, H,  \mathbb{P}, r)$, where $\cS$ is the (possibly uncountable) state space, $\cA$ is the (possibly uncountable) action space, $H $ is the number of steps in each episode, $\mathbb{P} = \{\mathbb{P}_h\}_{h=1}^{H}$ are the state transition probability distributions, and $r = \{r_h\}_{h=1}^{H}$ are the reward functions. For each $h \in [H]$,  $\mathbb{P}_h(\cdot \,|\, s, a)$ is the transition kernel over the next states if action $a$ is taken at state $s$ during the $h$-th time step of the episode.  
Also, $r_h:\cS\times\cA\rightarrow [0,1]$ is the deterministic reward function at step $h$.\footnote{We assume the reward function is deterministic for notational convenience. Our results can be straightforwardly generalized to the case when rewards are stochastic.} 

A policy $\pi$ is a collection of $H$ functions $\{\pi_h: \cS \rightarrow \Delta(\cA) \}_{h \in [H]}$ where $\Delta(\cA)$ denotes probability simplex over action space $\cA$. We denote by $\pi(\cdot \,|\, s)$ the action distribution of policy $\pi$ for state $s$, and by $\pi^*$ the optimal policy, which maximizes the value function defined below.

The value function $V_h^\pi:\cS \rightarrow \mathbb{R}$ at step $h \in [H]$ is the expected sum of remaining rewards until the end of the episode, received under  $\pi$ when starting from $s_h = s$,
\begin{equation*}
    V_h^{\pi}(s) = \mathbb{E}_\pi \Big[ \sum_{h'=h}^{H} r_{h'}(s_{h'}, a_{h'}) \,|\,  s_h = s   \Big].
\end{equation*}
The action-value function $Q_h^\pi:\cS\times\cA \rightarrow \mathbb{R}$ is defined as the expected sum of rewards given the current state and action when the agent follows policy $\pi$ afterwards,
\begin{equation*}
    Q_h^{\pi}(s,a) =   \mathbb{E}_\pi \Big[\sum_{h'=h}^{H} r_{h'}(s_{h'}, a_{h'}) \,|\,  s_h = s, a_h = a   \Big].
\end{equation*}
We denote $V_h^*(s)=V_h^{\pi^*}(s)$ and $Q_h^*(s,a) = Q_h^{\pi^*}(s,a)$. Moreover, to simplify notation, we denote $[\mathbb{P}_h V_{h+1}](s,a) = \mathbb{E}_{s' \sim \mathbb{P}_h(\cdot \,|\, s, a)}V_{h+1}(s')$.

Recall that value functions obey the Bellman equations:
\iftwo
    \begin{equation} \label{eq:bellman-eq} 
        \begin{alignedat}{2} 
        &Q^{\pi}_{h}(s, a) &&= (r_h + \mathbb{P}_h V^{\pi}_{h+1})(s, a),\\
        &V^{\pi}_{h}(s)    &&= \langle Q^{\pi}_{h}(s,\cdot), \pi_h(\cdot \,|\, s) \rangle_\cA,\\ 
        &V^{\pi}_{H+1}(s)  &&= 0. 
        \end{alignedat}
    \end{equation}
\else 
    \begin{align} \label{eq:bellman-eq} 
        Q^{\pi}_{h}(s, a) = (r_h + \mathbb{P}_h V^{\pi}_{h+1})(s, a) , \qquad   V^{\pi}_{h}(s) = \la Q^{\pi}_{h}(s,\cdot), \pi_h(\cdot \,|\, s) \ra_\cA,    
        \qquad V^{\pi}_{H+1}(s) = 0.  
    \end{align}
\fi
The aim of the agent is to learn the optimal policy by acting in the environment for $K$ episodes. Before starting each episode  $k \in [K]$, the agent chooses a policy $\pi^k$  and an adversary chooses the  initial state $s^k_1$. Then, at each time step $h \in [H]$, the agent observes $s^k_h \in \cS$, picks an action $a^k_h \in \cA$, receives a reward $r_h(s_h^k, a_h^k)$ and the environment transitions to the next state $s^k_{h+1}\sim \mathbb{P}_h(\cdot \,|\, s^k_h, a^k_h)$. The episode ends after the agent collects the $H$-th reward and reaches the state $s^k_{H+1}$. The suboptimality of an agent can be measured by its regret, the cumulative difference of optimal and achieved return, which after $K$ episodes is
 \begin{equation}
     \text{Regret}(K) = \sum_{k=1}^{K} \Big[ V_1^{*}(s_1^{k}) - V_1^{\pi^k} (s_1^{k})  \Big].
 \end{equation}

\textbf{Additional notations.}
The performance of function $f$ on dataset $\cD=\{(x_t,y_t)\}_{t\in[|\cD|]}$ is defined by 
$L(f \,|\, \cD) =\left(\sum_{t=1}^{|\cD|}(f(x_t)-y_t)^2\right)^{1/2}.$ The empirical $\ell_2$ norm of function $f$ on input set $\cZ=\{x_t\}_{t\in[|\cZ|]}$ is defined by
$\|f\|_{\cZ}=\left(\sum_{t=1}^{|\cZ|}f(x_t)^2\right)^{1/2}$. Given a function class $\cF\subseteq\{f:X \rightarrow \mathbb{R}\}$, we define the width function given some input $x$ as
$w(\cF,x)=\max_{f,f'\in\cF}f(x)-f'(x)$.

\section{Algorithm: LSVI-PHE}

\begin{algorithm}[t]
\caption{\label{Algorithm-GFA} $\mathcal{F}$-LSVI-PHE}

\begin{algorithmic}[1]
\STATE Set $M$ to be a fixed integer. 
\STATE \textbf{For} episode $k=1,2,\ldots, K$ \textbf{do}
\STATE \hspace{0.05in} Receive the initial state $s_1^k$.
\STATE \hspace{0.05in} Set $V_{H+1}^k(s)=0$ for all $s\in\mathcal{S}$.
\STATE \hspace{0.05in} \textbf{For} step {$h=H, H-1,\ldots, 1$} \textbf{do}  \label{line:pes-start}
\STATE \hspace{0.13in} \textbf{For} {$m=1,2,\ldots,M$} \textbf{do}
\STATE \hspace{0.31in} Sample i.i.d. Gaussian noise $\xi_{h,k}^{\tau,m}\sim \cN(0,\sigma_{h,k}^2)$.
\STATE \hspace{0.31in} Perturbed dataset: $\Tilde{\mathcal{D}}_{h}^{k,m}\leftarrow\{(s_{h}^\tau,a_{h}^\tau, r_{h}^\tau+\xi_{h,k}^{\tau,m}$ \label{line:perturbed dataset}
\STATE \hspace{0.31in} $+V_{h+1}^k(s_{h+1}^\tau))\}_{\tau\in[k-1]}.$
\STATE \hspace{0.31in} Set $\Tilde{f}_h^{k,m}\leftarrow \argmin_{f\in\mathcal{F}}L(f \,|\, \Tilde{\mathcal{D}}_h^{k,m})^2+\lambda \Tilde{R}(f)$.
\STATE \hspace{0.31in} Set $Q_h^{k,m}(\cdot,\cdot)\leftarrow \Tilde{f}_h^{k,m}(\cdot,\cdot)$.
\STATE \hspace{0.13in} Set $Q_h^{k}(\cdot,\cdot)\leftarrow\min\{\max_{m\in[M]}\{Q_h^{k,m}(\cdot,\cdot)\},$
\STATE \hspace{0.33in} $H-h+1\}$.
\STATE \hspace{0.13in} Set $V_h^k(\cdot)\leftarrow\max_{a\in \mathcal{A}}Q_h^k(\cdot,a)$ and 
\STATE \hspace{0.33in} $\pi_h^k(\cdot)\leftarrow\argmax_{a\in\mathcal{A}}Q_h^k(\cdot,a)$.
\STATE \hspace{0.05in} \textbf{For} step {$h=1, 2, \ldots, H$} \textbf{do}  \label{line:pis-start}

\STATE \hspace{0.23in}  Take  action  $a^k_{h} \leftarrow \argmax_{a \in \cA} Q_h^{k}(s_h^k,a)$.
\STATE \hspace{0.23in}  Observe reward $r^k_{h}(s_h^k,a_h^k)$, get next state $s^k_{h+1}$.  \label{line:pis-end}
\end{algorithmic}
\end{algorithm}

In this section, we lay out our algorithm LSVI-PHE\footnote{Here PHE stands for Perturbed History Exploration.} (Algorithm~\ref{Algorithm-GFA}), an optimistic modification of RLSVI, where the optimism is realized by, what we will call, optimistic reward sampling. 
To describe our algorithm and facilitate its analysis in Section ~\ref{Section:Theoretical Analysis}, we first define the perturbed least squares regression. We add noises on the regression target and the regularizer to achieve enough randomness in all directions of the regressor.

\begin{definition}
[Perturbed Least Squares]\label{def:perturbed-least-squares} Consider a function class $\mathcal{F}: X \rightarrow \mathbb{R}$. For an arbitrary dataset $\cD=\{(x_i,y_i)\}_{i=1}^n$, a regularizer $R(f) = \sum_{j=1}^D p_j(f)^2$ where $p_j(\cdot)$ are functionals, 
and positive constant $\sigma$, the perturbed dataset and perturbed regularizer are defined as
\begin{equation*}
    \Tilde{\cD}_\sigma=\{(x_{i},y_i+\xi_i)\}_{i=1}^n, \ \ \Tilde{R}_\sigma(f)=\sum_{j=1}^D [p_j(f)+\xi'_j]^2,
\end{equation*}
where $\xi_i$ and $\xi_j'$ are i.i.d. zero-mean Gaussian noises with variance $\sigma^2$. For a loss function $L$, the corresponding perturbed least squares regression solution is
$$\Tilde{f}_\sigma=\argmin_{f\in\cF} L(f \,|\, \Tilde{\mathcal{D}}_\sigma)^2+\lambda \Tilde{R}_\sigma(f).$$
\end{definition}
Within each episode $k \in [K]$, at each time-step $h$, we perturb the dataset by adding zero mean random Gaussian noise to the reward in the replay buffer $\{(s_h^\tau, a_h^\tau, r_h^\tau) \}_{\tau \in [k-1]}$ and the regularizer before we solve the perturbed regularized least-squares regression. At each time step $h$, we repeat the process for $M$ (to be specified in Section~\ref{Section:Theoretical Analysis})  times and use the maximum of the regressor as the optimistic estimate of the state-action value function. Concretely, we set $V_{H+1}^k = 0$ and calculate $Q_{H}^k, Q_{H-1}^k, \ldots, Q_{1}^k$ iteratively as follows. For each $h\in[H]$ and $m \in [M]$, we solve the following perturbed regression problem,
\begin{equation}\label{eq:perturbed-regression-problem}
  \Tilde{f}_h^{k,m}\leftarrow \argmin_{f\in\mathcal{F}}L(f \,|\, \Tilde{\mathcal{D}}_h^{k,m})^2+\lambda \Tilde{R}(f).
\end{equation}
We set $Q_h^{k,m}(\cdot, \cdot) =\Tilde{f}_h^{k,m}(\cdot, \cdot)$ and define 
\begin{equation}
    Q_h^{k}(\cdot,\cdot)= \min\{\max_{m\in[M]}\{Q_h^{k,m}(\cdot,\cdot)\},H-h+1\}.
\end{equation}
We then choose the greedy policy with respect to $Q_h^k$ and collect a trajectory data for the $k$-th episode. We repeat the procedure until all the $K$ episodes are completed.

\subsection{LSVI-PHE with Linear Function Class}
We now present LSVI-PHE when we consider linear function class (see Algorithm~\ref{Algorithm-linear MDP reward perturbation}). In this case, the following proposition shows that, adding scalar Gaussian noise to the reward is equivalent to perturbing the least-squares estimate using $d$-dimensional multivariate Gaussian noise. 

\begin{proposition}\label{Proposition:noise}
In line 9 of Algorithm~\ref{Algorithm-linear MDP reward perturbation}, conditioned on all the randomness except $\{\epsilon^{k,i,j}_{h}\}_{(i,j) \in [k-1]\times[M]}$ and $\{\xi^{k,j}_{h}\}_{j \in [M]}$, the estimated parameter $\Tilde{\theta}^{k,j}_h$ satisfies
$$\Tilde{\theta}^{k,j}_h - \Hat{\theta}^{k,j}_h\sim N(0,\sigma^2(\Lambda_h^k)^{-1}),$$
where $\Hat{\theta}^{k,j}_h=(\Lambda_h^k)^{-1}(\sum_{\tau=1}^{k-1}[r_h^\tau+V_{h+1}^k(s_{h+1}^\tau)]\phi(s_h^\tau,a_h^\tau))$ is the unperturbed regressor.
\end{proposition}

Intuitively, adding a zero-mean multivariate Gaussian noise on the parameter $\Hat{\theta}_h^k$ can guarantee that $\Tilde{Q}_h^k$ is optimistic with constant probability. By repeating this procedure multiple times, this constant probability can be amplified to arbitrary high probability.

\begin{algorithm}[t]
\caption{\label{Algorithm-linear MDP reward perturbation} LSVI-PHE with Linear function class}

\begin{algorithmic}[1]
\STATE Set $M$ to be a fixed integer.
\STATE \textbf{For} episode $k=1,2,\ldots, K$ \textbf{do}
\STATE \hspace{0.05in} Receive the initial state $s_1^k$.

\STATE \hspace{0.05in} \textbf{For} step {$h=H, H-1,\ldots, 1$} \textbf{do}  \label{line:linear-pes-start}

\STATE \hspace{0.13in} $\Lambda^k_h \leftarrow \sum_{\tau=1}^{k-1}\phi(s_h^\tau,a_h^\tau)\phi(s_h^\tau,a_h^\tau)^\top +  \lambda{I}$.\label{line:lambda}\vspace{0.054in}
\STATE \hspace{0.13in} Sample i.i.d. $\{\epsilon^{k,\tau,j}_{h}\}_{(\tau,j) \in [k-1]\times[M]} \sim \cN(0,\sigma^2) $.\label{line:noise-sample}
\STATE \hspace{0.13in} Sample i.i.d. $\{\xi^{k,j}_{h}\}_{j \in [M]} \sim \cN(0,\sigma^2\lambda I_d) $.\label{line:noise vector sample}

\STATE \hspace{0.13in} $\rho_h^{k,j} \leftarrow \sum_{\tau=1}^{k-1}\left( [ r^\tau_{h} + V_{h+1}^{k}(s_{h+1}^\tau)+\epsilon^{k,\tau,j}_{h}] \phi(s_h^\tau,a_h^\tau) \right)$.

\STATE  \hspace{0.13in} $\Tilde{\theta}^{k,j}_h \!\leftarrow (\Lambda^{k}_h)^{-1} (\rho_h^k +  \xi^{k,j}_{h})$. \label{line:theta-tilde}

\STATE \hspace{0.13in} $\Tilde{Q}^{k,j}_h(\cdot,\cdot) \leftarrow  \phi(\cdot, \cdot)^\top\Tilde{\theta}^{k,j}_h$ for $j \in [M]$.
\label{line:get_Q_tilde_perturbed}\vspace{0.054in}
\STATE \hspace{0.13in} $Q^k_{h}(\cdot,\cdot) \leftarrow \min \{\max_{j \in [M]} \Tilde{Q}^{k,j}_h(\cdot,\cdot), H -h +1\}^+$  
\STATE \hspace{0.13in} $V^k_{h}(\cdot) \leftarrow \max_{a \in \cA} Q_{h}^{k}(\cdot,a)$.
~\label{Alg:min-max-for-Q}

\STATE \hspace{0.05in} \textbf{For} step {$h=1, 2, \ldots, H$} \textbf{do}  \label{line:linear-pis-start}

\STATE \hspace{0.13in}  Take  action  $a^k_{h} \leftarrow \argmax_{a \in \cA} Q_h^{k}(s_h^k,a)$.
\STATE \hspace{0.15in}  Observe reward $r^k_{h}(s_h^k,a_h^k)$, get next state $s^k_{h+1}$.  \label{line:linear-pis-end}
\end{algorithmic}
\end{algorithm}

\section{Theoretical Analysis}\label{Section:Theoretical Analysis}

For the analysis we will need the concept of the \textit{eluder dimension} due to \cite{russo2013eluder}. Let $\mathcal{F}$ be a set of real-valued functions with domain $\mathcal{X}$. For $f \in \cF, x_1,...,x_t \in \mathcal{X}$, introduce the notation $f|_{(x_1,...,x_t)} = (f(x_1),...,f(x_t))$. We say that $x \in \mathcal{X}$ is $\epsilon$-independent of $x_1,...,x_t \in \mathcal{X}$ given $\cF$ if there exists $f,f' \in \cF$ such that $||(f-f')|_{(x_1,...,x_t)}||_2 \leq \epsilon$ while $f(x) - f'(x) > \epsilon$.

\begin{definition}
[Eluder dimension, \cite{russo2013eluder}]\label{definition:eluder-dimension} The eluder dimension $\text{dim}_\mathcal{E}(\cF,\epsilon)$ of $\cF$ at scale $\epsilon$ is the length of the longest sequence $(x_1,...,x_n)$ in $\mathcal{X}$ such that for some $\epsilon' \geq \epsilon$, for any $2 \leq t \leq n$, $x_t$ is $\epsilon'$-independent of $(x_1,...,x_{t-1})$ given $\cF$.
\end{definition}

For a more detailed introduction of eluder dimension, readers can refer to \citep{russo2013eluder,osband2014model,wang2020reinforcement,ayoub2020model}.

\subsection{Assumptions for General Function Approximation}

For our general function approximation analysis, we make a few assumptions first. To emphasize the generality of our assumptions, in Section~\ref{subsection:assumption-linear-function-class}, we show that our assumptions are satisfied by linear function class.

Our algorithm (Algorithm~\ref{Algorithm-GFA}) receives a function class $\cF \subseteq \{f: \cS \times \cA \rightarrow [0, H]\}$ as input and furthermore, similar to \cite{wang2020reinforcement,ayoub2020model}, we assume that for any $V:\cS\rightarrow [0,H]$, upon applying the Bellman backup operator, the output function lies in the function class $\cF$. Concretely, we have the following assumption.

\begin{assumption}\label{assumption:realizibility}
For any $V: \cS \rightarrow [0,H]$ and for any  $h\in[H]$, $r_h + P_hV \in \cF$, i.e. there exists a function $f_{V} \in \cF$ such that for all $(s,a) \in \cS\times \cA$ it satisfies
\begin{equation}\label{Eq:Assumption realizibility}
    f_V(s,a) = r_h(s,a) + P_hV( s,a).
\end{equation}
\end{assumption}

We emphasize that many standard assumptions in the RL theory literature such as tabular MDPs \cite{jaksch2010near,jin2018q} and Linear MDPs \cite{yang2019sample, jin2019provably} are special cases of Assumption~\ref{assumption:realizibility}.
In the appendix, we consider a misspecified setting and show that even when \eqref{Eq:Assumption realizibility} holds approximately, Algorithm~\ref{Algorithm-GFA} achieves provable regret bounds.

We further assume that our function class has bounded covering number.
\begin{assumption}\label{Assumption:bounded covering number}
For any $\varepsilon > 0$, there exists an $\varepsilon$-cover $\mathcal{C}(\cF, \varepsilon)$ with bounded covering number $\mathcal{N}(\cF,\varepsilon)$.
\end{assumption}

 Next we define anti-concentration width, which is a function of the function class $\cF$, dataset $\cD$ and noise variance $\sigma^2$.

\begin{definition}
[Anti-concentration Width Function]\label{def:anticoncentration width} For a loss function $L(\cdot\,|\,\cdot)$ and dataset $\mathcal{D}$, let $\Hat{f}=\argmin_{f\in\cF} L(f \,|\, \mathcal{D})^2 + \lambda R(f)$ be the regularized least squares solution and $\Tilde{f}_\sigma=\argmin_{f\in\cF} L(f \,|\, \Tilde{\mathcal{D}}_\sigma)^2+\lambda \Tilde{R}_\sigma(f)$ be the perturbed regularized least-squares solution. For a fixed $v \in (0,1)$, let $g_{\sigma}:X \rightarrow \mathbb{R}$ be a function  such that for any input $x$:
$$g_{\sigma}(x)=\sup\left\{g \in \mathbb{R}:\mathbb{P}\left(\Tilde{f}_\sigma(x)\geq\Hat{f}(x)+g\right)\geq v\right\}.$$
We call $g_\sigma(\cdot)$ the anti-concentration width function.
\end{definition}

In words, $g_\sigma(\cdot)$ is the largest value some $g\in \mathbb{R}$ can take such that the probability that $\tilde{f}_\sigma$ is greater than $\hat{f} + g$ is at least $v$.

We assume that for a concentrated function class, there exists a $\sigma$ such that the anti-concentration width is larger than the function class width.

\begin{assumption}
[Anti-concentration]\label{assumption:anticoncentration} Given the input $X = \{x_{i}\}_{i=1}^n$ of dataset $\cD$ and some arbitrary positive constant $\beta$, we define a function class $\cF_{X,\beta}=\{f:\|f-\Hat{f}\|_{X}^2+\lambda R(f-\Hat{f})\leq\beta\}$. We assume that there exists a $\sigma$ such that
$$g_{\sigma'}(x)\geq w(\cF_{X,\beta},x),$$
for all inputs $x$ and $\sigma'\geq\sigma$.
\end{assumption}

This assumption guarantees that the randomized perturbation over the regression target has large enough probability of being optimistic. This assumption is satisfied by the linear function class. For more details, see Section~\ref{subsection:assumption-linear-function-class}.

Our next assumption is on the regularizer function $R(\cdot)$.

\begin{assumption}
[Regularization]\label{assumption:regularization} We assume that our regularizer $R(\cdot)$ has several basic properties.
\begin{itemize}
    \item $R(f)+R(f')\geq c R(f+f')$ for some positive constant $c>0$, for all $ f,f'\in\cF$.
    \item $R(f)=R(-f)\geq0$, for all $ f\in\cF$.
    \item For any $V:\cS\rightarrow[0,H]$, $R(r+PV)\leq B$ for some constant $B \in \mathbb{R}$.
\end{itemize}
\end{assumption}

Here, the first property is nothing but a variation of triangle inequality. The second property is a symmetry property which is natural for norms. Both these properties are satisfied by commonly used regularizers such as $\ell_0$, $\ell_1$ or $\ell_2$ norms. The last property is a boundedness assumption. For the case of $\ell_0$ norm $B$ takes the value of the dimension of the space. Moreover, along with the most commonly used (weighted) $\ell_2$ regularizer, many other regularizers also satisfy this property.

Our final assumption is regarding the  boundedness  of the eluder dimension of the function class.

\begin{assumption}[Bounded Function Class]\label{assumption:bounded_function_class}
For any $V:\cS\rightarrow[0,H]$ and any $\cZ \in (\cS \times \cA)^\mathbb{N}$, let $\cF'$ be a subset of function class $\cF$, consisting of all $f\in \cF$ such that
\begin{equation*}
    \|f-\upsilon\|^2_{\cZ}+\lambda R(f-\upsilon) \leq \beta,
\end{equation*}
where $v=r + PV$. We assume that $\cF'$ has bounded eluder dimension.
\end{assumption}

Note that in \citet{wang2020reinforcement}, they assume that the eluder dimension of the whole function class $\cF$ is bounded. In contrast, ours is a weaker assumption since we only assume a subset $\cF'$ to have a bounded eluder dimension.

In the following section, we show that the linear function class and ridge regularizer  satisfy all the above assumptions.

\subsubsection{Linear Function Class}\label{subsection:assumption-linear-function-class}

First, we recall the standard linear MDP definition which was introduced in ~\cite{yang2019sample, jin2019provably}.

\begin{definition}[Linear MDP, \cite{yang2019sample, jin2019provably}]\label{definition-linear-MDP}
We consider a linear Markov decision process, $\mathrm{MDP}(\cS, \cA, H, \mathbb{P}, r)$ with a feature map $\phi: \cS \times \cA \rightarrow \mathbb{R}^d$, where for any $(h,k) \in [H] \times [K]$, there exist $d$ unknown (signed) measures $\mu_h = (\mu_h^{(1)},\cdots, \mu_h^{(d)})$ over $\cS$ and an unknown vector $w_h \in \mathbb{R}^d$, such that for any $(s,a) \in \cS \times \cA$, the following holds:
\begin{equation*}
    \mathbb{P}_h(s'| s,a) = \langle \phi(s,a), \mu_h(s')\rangle, \quad r_h(s,a) = \langle \phi(s,a), w_h \rangle.
\end{equation*}
Without loss of generality, we assume, for all $(s,a) \in \cS \times \cA$, $\|\phi(s,a)\| \leq 1$, and for all $h \in [H]$, $\|w_h\|  \leq \sqrt{d}$ and $\|\mu_h(\cS) \| \leq \sqrt{d}$. 
\end{definition}

Consider a fixed  episode $k$ and step $h$. We define $\cF=\{f_\theta:f_\theta(s,a)=\phi(s,a)^\top\theta\}$ where $\theta \in \mathbb{R}^d$, $\cD = \{(s_h^\tau,a_h^\tau,r_h^\tau)\}_{\tau \in [k-1]}$, and $R(f_\theta)=\|\theta\|^2=\sum_{j=1}^d p_j(f_\theta)^2$ where $p_j(f_\theta)=e_j^\top\theta$ with $e_j$ being the $j$-th standard basis vector. It is well known that linear function class satisfies Assumption~\ref{assumption:realizibility} in linear MDP \citep{yang2019sample,jin2019provably}.
We set $\Hat{f}=\argmin_{f\in\cF} L(f\,|\,\cD)^2+\lambda R(f)$ to be $f_{\Hat{\theta}}$. Then we have 
\begin{align*}
    \Hat{\theta}&=\argmin_\theta \sum_{\tau=1}^{k-1} (\phi(s_h^\tau,a_h^\tau)^\top\theta-r_h^\tau)^2+\lambda\|\theta\|^2 \\ &=(\Lambda^k_h)^{-1}\sum_{\tau=1}^{k-1} r_h^\tau\phi(s_h^\tau,a_h^\tau),
\end{align*}

where $\Lambda_h^k=\sum_{\tau=1}^{k-1}\phi(s_h^\tau,a_h^\tau)\phi(s_h^\tau,a_h^\tau)^\top+\lambda I$. Similarly we set $f_{\Tilde{\theta}}=\Tilde{f}_\sigma=\argmin_{f\in\cF} L(f\,|\,\tilde{\cD}_\sigma)^2+\lambda \Tilde{R}_\sigma(f)$. Then we have 
\begin{align*}
  \Tilde{\theta} &=(\Lambda^k_h)^{-1}\sum_{\tau=1}^{k-1} (r_h^\tau+\xi_\tau)\phi(s_h^\tau,a_h^\tau)+(\Lambda^k_h)^{-1}\sum_{j=1}^d \xi'_j e_j \\ &\sim \cN(\Hat{\theta},\sigma^2(\Lambda^k_h)^{-1}).  
\end{align*}

For Definition~\ref{def:anticoncentration width}, we set $v=\Phi(-1)$. Using the anti-concentration property of Gaussian distribution, it is straightforward to show that for any $(s,a) \in \cS\times\cA$:

\begin{equation*}
    \mathbb{P}\left(f_{\Tilde{\theta}}(s,a)\geq f_{\Hat{\theta}}(s,a)+\sigma\|\phi(s,a)\|_{(\Lambda^k_h)^{-1}}\right)=v.
\end{equation*}
So we have $g_\sigma(s,a)\geq\sigma\|\phi(s,a)\|_{(\Lambda^k_h)^{-1}}$ from Definition~\ref{def:anticoncentration width}.

For Assumption~\ref{assumption:anticoncentration}, the function class $\cF_{\cD,\beta}=\{f:L(f-\Hat{f} \,|\,\cD)^2+\lambda R(f-\Hat{f})\leq\beta\}$ is equivalent to $\Theta_{\cD,\beta}=\{\theta:(\theta-\Hat{\theta})^\top\Lambda_h^k(\theta-\Hat{\theta})\leq\beta\}.$ So the width on the state-action pair $(s,a)$ is $2\sqrt{\beta}\|\phi(s,a)\|_{(\Lambda^k_h)^{-1}}.$ If we set $\sigma=2\sqrt{\beta}$, we have 
$$g_\sigma(s,a)\geq w(\cF_{\cD,\beta},s,a).$$

For Assumption~\ref{assumption:regularization}, as $R(f_\theta)=\|\theta\|^2$ is a $\ell_2$ norm function, the first two properties are direct to show with constant $c=1/2$. For the third property, we have that
$$g(s,a)=r(s,a)+P(s,a)V=\phi(s,a)(w+\sum_{s'}V(s')\mu(s')).$$
So we have $g=g_\theta$ where $\theta=w+\sum_{s'}V(s')\mu(s')$ and $\|\theta\|^2\leq 2Hd$.

For Assumption~\ref{assumption:bounded_function_class}, we set $\theta_f:f=f_{\theta_f}$, $\theta_v:v=f_{\theta_v}$ and $\Theta_{\cF'}=\{\theta:f_\theta\in\cF'\}$ to be the parameterization. From Assumption~\ref{assumption:regularization}, we have $\|\theta_v\|^2\leq2Hd$. In addition, we have $\lambda R(f-v)=\lambda\|\theta_f-\theta_v\|^2\leq\beta$. Then we have 
\begin{align*}
    \Theta_{\cF'} &\subseteq\{\theta_f:\|\theta_f-\theta_v\|^2\leq \beta/\lambda,\|\theta_v\|^2\leq 2Hd\}\\
    &=\{\theta_f:\|\theta_f\|^2\leq2\beta/\lambda+4Hd\}.
\end{align*}
As shown in \citep{russo2013eluder}, this $\cF'$ has eluder dimension $\Tilde{O}(d)$.

\subsection{Regret bound for General Function Approximation}

First, we specify our choice of the noise variance $\sigma^2$ in the algorithm. We prove certain concentration properties of the regularized regressor $\Hat{f}_h^k$ so that the condition in Assumption~\ref{assumption:anticoncentration} holds. Thus we can choose an appropriate $\sigma$ such that the Assumption~\ref{assumption:anticoncentration} is satisfied. A more detailed description is provided in the appendix.

Our first lemma is about  the concentration of the regressor. A similar argument appears in \cite{wang2020reinforcement} but their result does not include regularization, which is essential in our randomized algorithm to ensure exploration in all directions.

\begin{lemma} [Informal Lemma on Concentration]
Under Assumptions~\ref{assumption:realizibility}, \ref{Assumption:bounded covering number}, \ref{assumption:anticoncentration}, \ref{assumption:regularization}, and~\ref{assumption:bounded_function_class}, let $\cF_h^{k,m}=\{f\in\cF| \lVert f - \widetilde{f}_h^{k,m} \rVert_{\cZ_h^k}^2 + \lambda R(f-\widetilde{f}_h^{k,m}) \leq \beta(\cF,\delta)\}$, where $\mathcal{Z}_h^k = \{(s_h^\tau, a_h^\tau)\}_{\tau \in [k-1]}$, and $$\beta(\cF,\delta)= \Tilde{O}\left((H+\sigma)^2\log\mathcal{N}\left(\cF,1/T\right)\right).$$
With high probability, for all $(k,h,m)\in[K]\times[H]\times[M]$, we have
$$r_h(\cdot,\cdot)+P_hV_{h+1}^k(\cdot,\cdot)\in\cF_h^{k,m}.$$
\end{lemma}

This lemma shows that the perturbed regularized regression still enjoys concentration.

Our next lemma shows that LSVI-PHE is optimistic with high probability.

\begin{lemma} [Informal Lemma on Optimism]
Let
\begin{equation*}
M=\ln\left(\frac{T|\cS||\cA|}{\delta}\right)/\ln\left(\frac{1}{1-v}\right).
\end{equation*}
With probability at least $1-\delta$, for all $(s,a,h,k) \in \cS\times\cA\times[H]\times[K]$, we have
$$Q_h^*(s,a)\leq Q_h^k(s,a).$$
\end{lemma}
With optimism, the regret is known to be bounded by the sum of confidence width \cite{wang2020reinforcement}. As Assumption~\ref{assumption:bounded_function_class} assumes that all the confidence region is in a bounded function class in the measure of eluder dimension, we can adapt proof techniques from \citep{wang2020reinforcement} and prove our final result.

\begin{theorem} [Informal Theorem]\label{thm:main_gfa_regret}

Under Assumptions~\ref{assumption:realizibility}, \ref{Assumption:bounded covering number}, \ref{assumption:anticoncentration}, \ref{assumption:regularization}, and~\ref{assumption:bounded_function_class}, with high probability, Algorithm~\ref{Algorithm-GFA} achieves a regret bound of 
$$\mathrm{Regret}(K)\leq\Tilde{O}\left(\sqrt{\mathrm{dim_\mathcal{E}}(\cF,1/T)\beta(\cF,\delta)HT}\right),$$
where
$$\beta(\cF,\delta)= \Tilde{O}\left((H+\sigma)^2\log\mathcal{N}\left(\cF,1/T\right)\right).$$

\end{theorem}
The theorem shows that our algorithm enjoys sublinear regret and have polynomial dependence on the horizon $H$, noise variance $\sigma^2$ and eluder dimension $\mathrm{dim_\mathcal{E}}(\cF,1/T)$, and have logarithmic dependence on the covering number of the function class $\mathcal{N}\left(\cF,1/T\right)$. 

\subsection{Regret bound for linear function class}

Now we present the regret bound for Algorithm~\ref{Algorithm-linear MDP reward perturbation} under the assumption of linear MDP setting. In the appendix, we provide a simple yet elegant proof of the regret bound.

\begin{theorem} \label{thm:theorem-linear MDP} Let  $M = d\log(\delta/9)/\log \Phi(1)$, $\sigma = \Tilde{O}(H\sqrt{d})$, and $\delta \in (0,1]$. Under linear MDP assumption from Definition~\ref{definition-linear-MDP}, the regret of Algorithm~\ref{Algorithm-linear MDP reward perturbation} satisfies
$$\text{Regret}(T) \leq \Tilde{\mathcal{O}}(d^{3/2} H^{3/2} \sqrt{T}),$$
with probability at least $1-\delta$.
\end{theorem}

\begin{remark}
Under linear MDP assumption, this regret bound is at the same order as the LSVI-UCB algorithm from \citep{jin2019provably} and $\sqrt{dH}$ better than the state-of-the-art TS-type algorithm \citep{zanette2019frequentist}. The only work that enjoys a $\sqrt{d}$ better regret is \citep{zanette2020learning}, which requires solving an intractable optimization problem. 
\end{remark}

\begin{remark}
Along with being a competitive algorithm in statistical efficiency, we want to emphasize that our algorithm has good computational efficiency. LSVI-PHE with linear function class only involves linear programming to find the greedy policy while LSVI-UCB \citep{jin2019provably} requires solving a quadratic programming. The optimization problem in OPT-RLSVI \citep{zanette2019frequentist} is hard too because the Q-function there is a piecewise continuous function and in one piece, it includes the product of the square root of a quadratic term and a linear term.

\end{remark}

\section{Numerical Experiments}
We run our experiments on RiverSwim \cite{strehl2008analysis}, DeepSea \cite{osband2016generalization} and sparse MountainCar \citep{brockman2016openai} environments as these are considered to be hard exploration problems where $\varepsilon$-greedy is known to have poor performance. For both RiverSwim and DeepSea experiments, we make use of linear features.  The objective here is to compare an exploration method that randomizes the targets in the history (LSVI-PHE) with an exploration method that computes upper confidence bounds given the history (LSVI-UCB) \cite{jin2019provably, cai2019provably}. For the continous control MountainCar environment, we use neural-network as function approximator to implement LSVI-PHE. The objective here is to compare deep RL variant of LSVI-PHE against other popular deep RL algorithms specifically designed to tackle exploration task.

\subsection{Measurements} We plot the per episode return of each algorithm to benchmark their performance. As the agent begins to act optimally the per episode return begins to converge to the optimal, or baseline, return. The per episodes returns are the sum of all the rewards obtained in an episode. We also report the performance of LSVI-PHE when $\sigma^2$ is fixed and $M$ varies. 

\subsection{Results for RiverSwim} A diagram of the RiverSwim environment is shown in the Appendix. RiverSwim consists of $\cS$ states lined up in a chain. The agent begins in the leftmost state $s_1$ and has the choice of swimming to the left or to the right at each state. The agent's goal is to maximize its return by trying to reach the rightmost state which has the highest reward. Swimming to the left, with the current, transitions the agent to the left deterministically. Swimming to the right, against the current, stochastically transitions the agent and has relatively high probability of moving right toward the goal state. However, because the current is strong there is a high chance the agent will stay in the current state and a low chance the agent will get swept up in the current and transition to the left. 
Thus, smart exploration is required to learn the optimal policy in this environment. 
We experiment with the variant of RiverSwim where $\cS = 12$ and $H = 40$. For this experiment, we swept over the exploration parameters in both LSVI-UCB \citep{jin2019provably} and LSVI-PHE and report the best performing run on a $12$ state RiverSwim. LSVI-UCB computes confidence widths of the following form $\beta \lVert \phi(s,a) \rVert_{\Sigma^{-1}}$ where $\phi(s,a) \in \mathbb{R}^d$ are the features for a given state-action pair and $\Sigma \in \mathbb{R}^{d \times d}$ is the empirical covariance matrix. We sweep over $\beta$ for LSVI-UCB and $\sigma^2$ for LSVI-PHE, where $M$ is chosen according to our theory (Theorem~\ref{thm:theorem-linear MDP}). We sweep over these parameters to speed up learning as choosing the theoretically optimal choices for $\beta$ and $\sigma^2$ often leads to a more conservative exploration policy which is slow to learn. As shown in Figure \ref{fig:riverswim_best_run_12}, the best performing LSVI-PHE achieves similar performance to the best performing LSVI-UCB on the 12 state RiverSwim environment. 

\begin{figure}[t]
    \centering
    \includegraphics[width=0.5\textwidth]{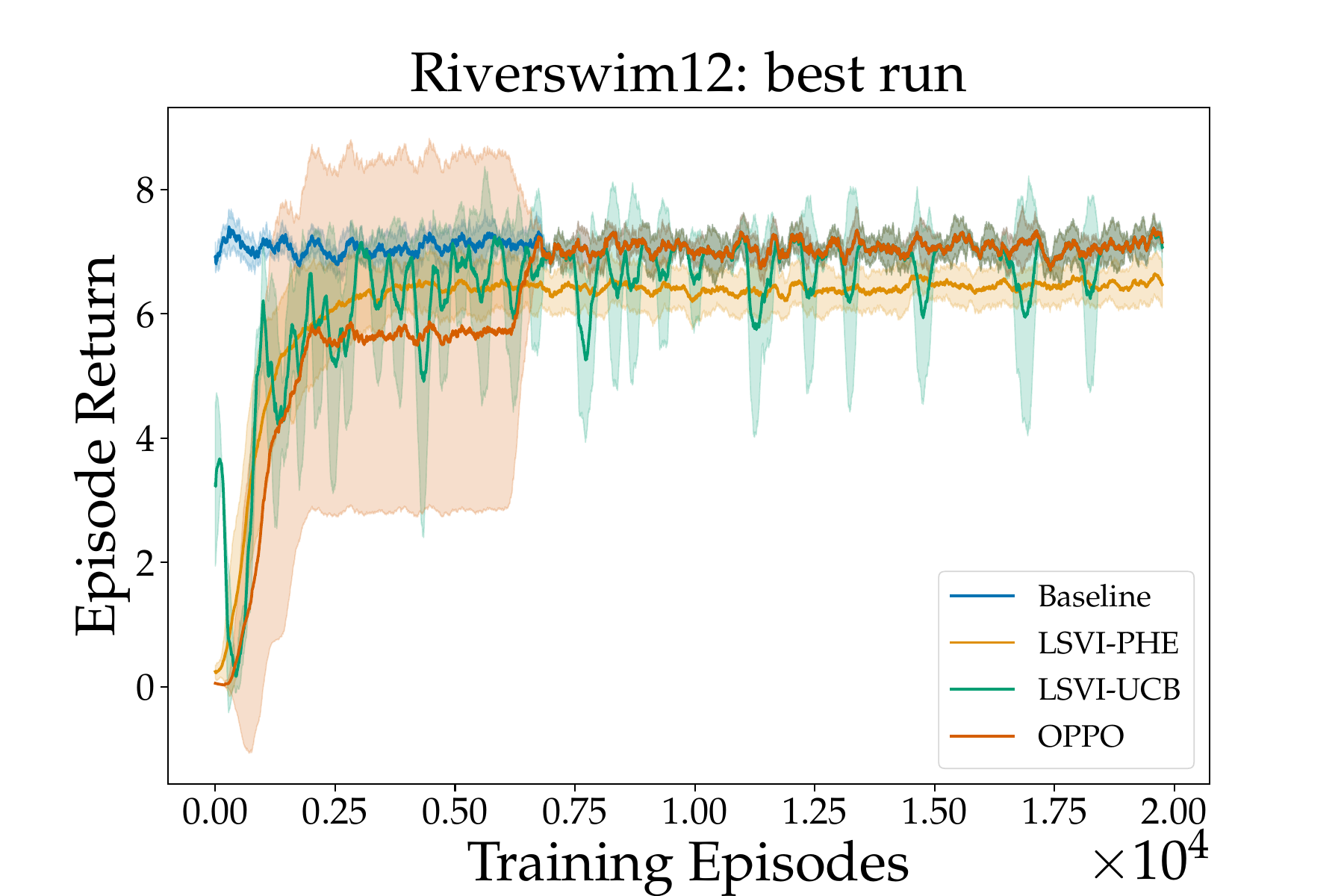}
    \caption{The results are averaged over 10 independent runs and error bars are reported for the regret plots. For this plot, $\beta = 5.0$ for LSVI-UCB and $\sigma^2 = 2\times 10^{-1}$ for LSVI-PHE.} 
    \label{fig:riverswim_best_run_12}
\end{figure}

\subsection{Results for DeepSea} DeepSea \cite{osband2016generalization} consists of $\cS = N \times N$ states arranged in a grid, where $N$ is the depth of the sea. The agent begins at the top leftmost state in the grid $s_1$ and has the choice of moving down and left or down and right at each state. Once the agent reaches the bottom of the sea it transitions back to state $s_1$. The agent's goal is to maximize its return by reaching the bottom right most state. The agent gets a small negative reward for transitioning to the right while no reward is given if the agent transitions to the left. Thus, smart exploration is required; otherwise the agent will rarely go right the necessary amount of time to reach the goal state. We run our experiments on a $10 \times 10$ DeepSea environment. As shown in Figure \ref{fig:deepsea_best_run_8}, the best performing LSVI-PHE achieves similar performance to the best performing LSVI-UCB on DeepSea. 
We also vary $M$ given a fixed $\sigma^2 = 5\times 10^{-4}$. As shown in Figure \ref{fig:deepsea_sweep_M_sig0.0005}, as we increase $M$, the performance of LSVI-PHE increases. 

\begin{figure}[t]
    \centering
    \includegraphics[width=0.5\textwidth]{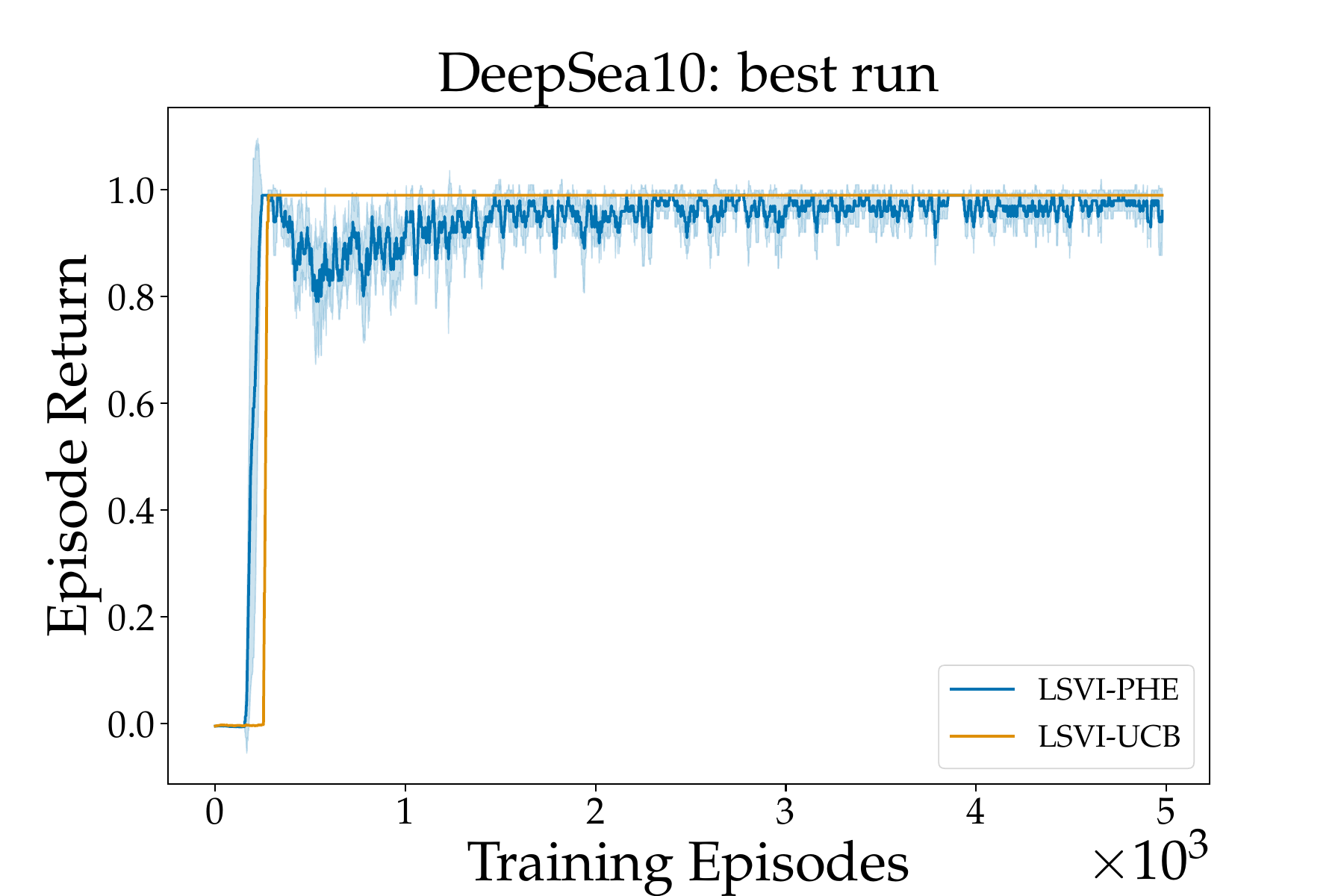}
    \caption{The results are averaged over 5 independent runs and error bars are reported for the return per episode plots. For this plot, $\beta = 5 \times 10^{-3}$ for LSVI-UCB and $\sigma^2 = 5 \times 10^{-5}$ for LSVI-PHE.}
    \label{fig:deepsea_best_run_8}
\end{figure}

\begin{figure}[t]
    \centering
    \includegraphics[width=0.50\textwidth]{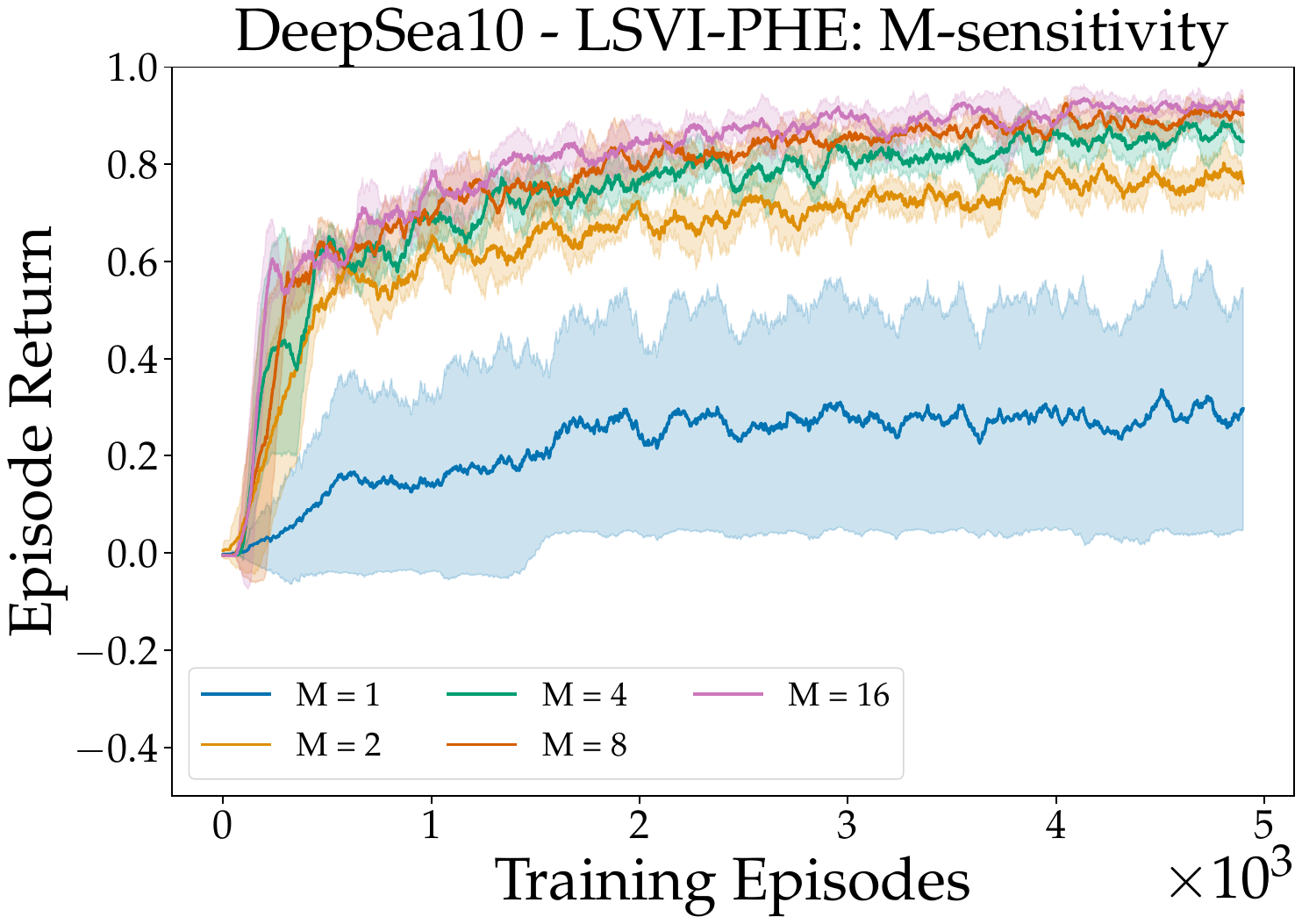}
    \caption{The results are averaged over 5 runs and error bars are reported for the return per episode plots. For this plot we fix $\sigma^2 = 5 \times 10^{-4}$.}
    \label{fig:deepsea_sweep_M_sig0.0005}
\end{figure}

These experiments on hard exploration problems highlight that we are able to simulate optimistic exploration, as in UCB, by perturbing the targets multiple times and taking the max over the perturbations to boost the probability of an optimistic estimate. If we are willing to sweep over $M$, the number of times we perturb the history, and $\sigma^2$, we can then get a faster algorithm that still performs well in practice. If we let $M=1$ and $\sigma^2=1$ then LSVI-PHE reduces to RLSVI and we would get the same performance as in \cite{osband2016generalization}. 

\subsection{Results for MountainCar}
We further evaluated LSVI-PHE on a continuous control task which requires exploration: sparse reward variant of continuous control MountainCar from OpenAI Gym \citep{brockman2016openai}. This environment consists of a $2$-dimensional continuous state space and a $1$-dimensional continuous action space $[-1,1]$. The agent only receives a reward of $+1$ if it reaches the top of the hill and everywhere else it receives a reward of $0$. We set the length of the horizon to be $1000$ and discount factor $\gamma = 0.99$.

For this setting, we compare four algorithms: LSVI-PHE, DQN with epsilon-greedy exploration, Noisy-Net DQN \citep{fortunato2017noisy} and Bootstrapped DQN \citep{osband2016deep}. Our experiments are based on the baseline implementations of \citep{Explorer}. As neural network, we used a multi-layer perceptron with hidden layers fixed to $[32,32]$. The size of the replay buffer was $10,000$. The weights of neural networks were optimized by Adam \citep{kingma2014adam} with gradient clip $5$. We used a batch size of $32$. The target network was updated every 100 steps. The best learning rate was chosen from $[10^{-3}, 5\times10^{-4}, 10^{-4}]$. For LSVI-PHE, we set $M = 8$ and we chose the best value of $\sigma$ from $[10^{-4}, 10^{-3}, 10^{-2}]$. Results are shown in Figure~\ref{fig:mountaincar}.

\begin{figure}[t]
    \centering
    \includegraphics[width=0.5\textwidth]{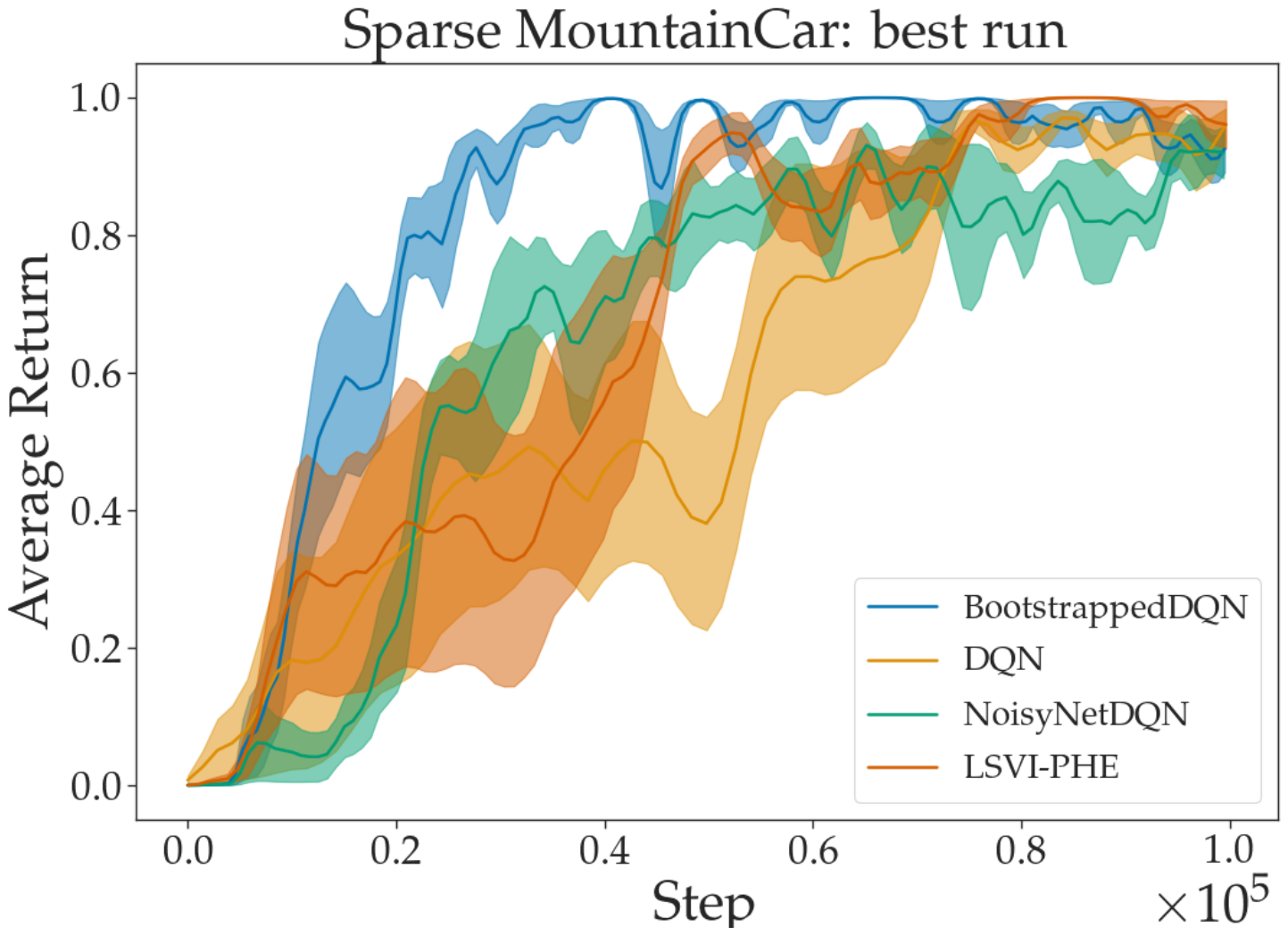}
    \caption{: Comparison of four algorithms on sparse MountainCar. The results are averaged over 5 independent runs and error bars are reported for the return per episode plots.}
    \label{fig:mountaincar}
\end{figure}

\section{Related Works}

\textbf{RL with Function Approximation.}
Many recent works have studied RL with function approximation, especially in the linear case \citep{jin2019provably,cai2019provably,zanette2019frequentist,zanette2020learning,wang2020reinforcement,ayoub2020model,foster2020instance,jiang2017contextual,sun2019model}. Under the assumption that the agent has access to a well-designed feature extractor, these works design provably efficient algorithms for linear MDPs and linear kernel MDPs. LSVI-UCB \cite{jin2019provably}, the first work with both polynomial runtime and polynomial
sample complexity with linear function approximation, has a regret of $\Tilde{\mathcal{O}}(\sqrt{d^3H^3T})$. The state-of-the-art regret bound is $\Tilde{\mathcal{O}}(Hd\sqrt{T})$, achieved by ELEANOR \citep{zanette2020learning}. However, ELEANOR needs to solve an optimization problem in each episode, which is computationally intractable. \citet{wang2019optimism} introduces a new expressivity condition named \textit{optimistic closure} for generalized linear function approximation under which they propose a variant of optimistic LSVI with regret bound $\Tilde{\mathcal{O}}(\sqrt{d^3H^3T})$. \citet{wang2020reinforcement,ayoub2020model} focus on online RL with general function approximation and their analysis is based on the eluder dimension \citep{russo2013eluder}. Other complexity measures of general function classes include disagreement coefficient \citep{foster2020instance}, Bellman rank \citep{jiang2017contextual} and Witness rank \citep{sun2019model}.

\textbf{Thompson Sampling.} Thompson Sampling ~\citep{thompson1933likelihood} was proposed almost a century ago and rediscovered several times. \citet{strens2000bayesian} was the first work to apply TS to RL. \citet{osband2013more} provides a Bayesian regret bound and \citet{agrawal2016near,ouyang2017learning} provide worst case regret bounds for TS. 

Randomized least-squares value iteration (RLSVI), proposed in \citet{osband2019deep}, uses random perturbations to  approximate the posterior. Recently, several works focussed on the theoretical analysis of RLSVI \citep{russo2019worst,zanette2020learning,agrawal2020improved}. \citet{russo2019worst} provides the first worst-case regret $\Tilde{O}(H^{5/2}S^{3/2}\sqrt{AT})$ for tabular MDP and \citet{agrawal2020improved} improves it to $\Tilde{O}(H^2S\sqrt{AT})$. \citet{zanette2019frequentist} proves $\Tilde{O}(H^2d^2\sqrt{T})$ regret bound for linear MDP. However, \citet{agrawal2020improved,zanette2019frequentist} both need to compute the confidence width as a warm-up stage, which is complicated and computationally costly.

\section{Conclusion}

In this work, we propose an algorithm LSVI-PHE for online RL with function approximation based on optimistic sampling. We prove the theoretical guarantees of LSVI-PHE and through experiments also demonstrate that it performs competitively against previous algorithms. We believe optimistic sampling provides a new provably efficient exploration paradigm in RL and it is practical in complicated real-world applications. We hope our work can be one step towards filling the gap between theory and application.

\section*{Acknowledgements}
The authors would like to thank Csaba Szepesvári for helpful discussions. HI thanks Qingfeng Lan for many insightful discussions regarding the experiments. The work was initiated when HI, ZY, ZW and LY was visiting the Simons Institute for the Theory of Computing at UC-Berkeley (Theory of Reinforcement Learning Program).
\bibliography{main}
\bibliographystyle{icml2021}

\newpage

\appendix
\onecolumn
\section{LSVI-PHE with General Function Approximations}

\subsection{Noise}\label{sec:noise}

In the section, we specify how to choose $\sigma$ in Algorithm~\ref{Algorithm-GFA}. Note that we use $\xi_{h,k}^{\tau,m}$ for the noise added in episode $k$, timestep $h$, data from episode $\tau<k$ and sampling time $m$. Similarly, $\xi'^{i,m}_{h,k}$ is for episode $k$, timestep $h$, regularizer $p_i(\cdot)$ and sampling time $m$. We set $\lambda=1$ in our algorithm. By Lemma \ref{Lemma:GFA_r_plus_PV_in_f_khm (unperturbed)}, there exists $\beta'(\cF,\delta)$ such that with probability at least $1-\delta$, for all $(k,h)\in[K]\times[H]$, we have
$$f_h^k(\cdot,\cdot):=r(\cdot,\cdot)+P_hV_{h+1}^k(\cdot,\cdot)\in\cF_h^{k},$$
where $\cF_h^k=\{f\in\cF\,|\,\|f-\Hat{f}_h^{k}\|_{\cZ_h^k}^2+R(f-\Hat{f}_h^{k})\leq \beta'(\cF,\delta)\}.$ By Assumption \ref{assumption:anticoncentration}, for each $\cF_h^k$, there exists a $\sigma_{h,k}$ such that 
$$g_{\sigma_{h,k}}(s,a)\geq w(\cF_h^k,s,a).$$
We define $\sigma=\max_{k\in[K],h\in[H]}\sigma_{h,k}$ to be the maximum standard deviation of the added noise.

\subsection{Concentration}

We first define few filtrations and good events that we will use in the proof of lemmas in this section.
\begin{definition}[Filtrations]\label{def:GFA_filtrations} 
We denote the $\sigma$-algbera generated by the set $\cG$ using $\sigma(\cG)$.  We define the following filtrations 
\begin{align*} 
    \cG^k &\mydef \sigma \left(\{(s_t^i, a_t^i, r_t^i)\}_{\{i,t\} \in [k-1]\times[H]}\,\ \bigcup \,\ \{\xi^{i,j}_{t,l}\}_{i\in[l],\{t,j,l\} \in [H]\times[M]\times[k-1]} \bigcup \,\ \{\xi'^{i,j}_{t,l}\}_{\{i,t,j,l\} \in [D]\times[H]\times[M]\times[k-1]}\right),\\
    \cG_{h,1}^k &\mydef \sigma \left(\cG^k \,\ \bigcup \,\ \{(s_t^k, a_t^k, r_t^k)\}_{t\in[h]}\bigcup \,\ \{\xi^{i,j}_{t,k}\}_{i\in[k],t \geq h,j\in[M]}\bigcup \,\ \{\xi'^{i,j}_{t,k}\}_{i\in[D],t \geq h,j\in[M]}\right),\\
    \cG_{h,2}^K &\mydef \sigma \left(\cG^k\,\  \bigcup \,\ \{(s_t^k, a_t^k, r_t^k)\}_{t\in[h]} \right).
\end{align*}
\end{definition}

\begin{definition}[Good events]\label{def:GFA_good-event}

For any $\delta > 0$, we define the following random events
\begin{align*}
    \mathcal{G}_{h}^k(\xi, \delta) &\mydef \Bigl\{ \max_{i\in[k],j \in [M]}\left|\xi^{i,j}_{h,k}\right| \leq \sqrt{\gamma_k(\delta)}\bigcap\max_{i\in[D],j \in [M]}|\xi'^{i,j}_{h,k}| \leq \sqrt{\gamma_k(\delta)}\Bigr\},\\
    \mathcal{G}(K, H, \delta) &\mydef \bigcap_{k \leq K}\bigcap_{h \leq H}\mathcal{G}^k_h(\xi,\delta),
\end{align*}
where $\gamma_k(\delta)$ is some constant to be specified in Lemma~\ref{lemma:GFA_good_event_prob}.
\end{definition}

\textbf{Notation: } To simplify our presentation, in the remaining part of this section, we always denote $\sqrt{\gamma_k}:=\sqrt{\gamma_k(\delta)}$.

The next lemma shows that the good event defined in Definition~\ref{def:GFA_good-event} happens with high probability. 
\begin{lemma}\label{lemma:GFA_good_event_prob}
For good event $\cG(K,H,\delta)$ defined in Definition \ref{def:GFA_good-event}, if we set $\sqrt{\gamma_k}=\Tilde{O}(\sigma)$, then it happens with probability at least $1-\delta$.
\end{lemma}

\begin{proof}
Recall that $\xi_{t,l}^{i,j}$ is a zero-mean Gaussian noise with variance $\sigma_{t,l}^2$. By the concentration of Gaussian distribution (Lemma \ref{lemma-gaussian-concentration}), with probability $1-\delta'$, we have
$$|\xi_{t,l}^{i,j}|\leq \sigma_{t,l} \sqrt{2\log(1/\delta')}\leq\sigma \sqrt{2\log(1/\delta')}.$$
The same result holds for $\xi'^{i,j}_{t,l}$. We complete the proof by setting $\delta'=\delta/(K+D)MHK$ and using union bound.
\end{proof}

In Definition~\ref{def:perturbed-least-squares}, for a regularizer $R(f) = \sum_{j=1}^D p_j(f)^2$, where $p_j(\cdot)$ are functionals, we defined the perturbed regularizer as $\Tilde{R}_\sigma(f)=\sum_{j=1}^D [p_j(f)+\xi'_j]^2$ with $\xi'_j$ being i.i.d. zero-mean Gaussian noise with variance $\sigma^2$. Note that in the algorithm, the variance of the noise for the regularizer is the same as the dataset, which is $\sigma_{h,k}^2$.
Recall from Assumption~\ref{assumption:regularization} that for any $V:\cS\rightarrow[0,H]$, our regularizer $R$ satisfies $R(r+PV)\leq B$ for some constant $B\in\mathbb{R}$.

Our next lemma establishes a bound on the perturbed estimate of a single backup.
\begin{lemma}\label{Lemma:GFA-Concentration} Consider a fixed $k \in [K]$ and a fixed $h \in [H]$. Let $\cZ_h^k=\{(s_{h}^\tau,a_{h}^\tau)\}_{\tau\in[k-1]}$ and $\Tilde{\mathcal{D}}_{h,V}^{k}=\{(s_{h}^\tau,a_{h}^\tau, r_{h}^\tau+\xi_{h}^{\tau}+V(s_{h+1}^\tau))\}_{\tau\in[k-1]}$. Define
$\Tilde{f}_{h,V}^{k}= \argmin_{f\in\mathcal{F}}\|f\|_{\Tilde{\mathcal{D}}_{h,V}^{k}}^2+\Tilde{R}(f).$ Conditioned on the good event $\cG(K,H,\delta)$, with probability at least $1-\delta$, for a fixed $V:\mathcal{S}\rightarrow[0,H]$ and any $V':\mathcal{S}\rightarrow[0,H]$ with $\|V'-V\|_\infty\leq1/T$, we have
\begin{align*}
&\left\|\widetilde{f}_{h,V'}(\cdot,\cdot)-r_h(\cdot,\cdot)-P_hV'( \cdot,\cdot)\right\|^2_{\mathcal{Z}_h^k}+ R\left(\widetilde{f}_{h,V'}(\cdot,\cdot)-r_h(\cdot,\cdot)-P_hV'( \cdot,\cdot)\right)\\\leq& c'\left[\left(H+1 + \sqrt{\gamma_k}\right)\sqrt{\log\left(2/\delta\right)+\log\mathcal{N}\left(\cF,1/T\right)}+\sqrt{B+\sqrt{\gamma_kBD}}\right]^2,
\end{align*}
for some constant $c'$. Here $B$ is the bound on the regularizer (Assumption \ref{assumption:regularization}) and $D$ is the number of regularizers (Definition \ref{def:perturbed-least-squares}). Define this event as $\cE_{h,V}(\delta)$.
\end{lemma}

\begin{proof}
Recall that for notational simplicity, we denote $[\mathbb{P}_h V_{h+1}](s,a) = \mathbb{E}_{s' \sim \mathbb{P}_h(\cdot \,|\, s, a)}V_{h+1}(s')$.
Now consider a fixed $V:\cS\rightarrow [0,H]$, and define
\begin{equation}
    f_V(\cdot, \cdot) := r_h(\cdot,\cdot) +  P_hV( \cdot, \cdot).
\end{equation}

For any $f \in \cF$, we consider $\sum_{\tau\in[k-1]}\chi_h^\tau(f)$ where $$\chi_h^\tau(f):=2(f(s_h^\tau,a_h^\tau)-f_V(s_h^\tau,a_h^\tau))(f_V(s_h^\tau,a_h^\tau)-r_h^\tau(s_h^\tau,a_h^\tau) - \xi_{h}^{\tau} -V(s_{h+1}^\tau)).$$

Recalling the definition of the filtration $\cG_{h,1}^\tau$ from Definition~\ref{def:GFA_filtrations}, we note
\begin{align*}
    \mathbb{E}[\chi_h^\tau(f)|\cG_{h,1}^\tau]&=\mathbb{E}[2(f(s_h^\tau,a_h^\tau)-f_V(s_h^\tau,a_h^\tau))(f_V(s_h^\tau,a_h^\tau)-r_h^\tau(s_h^\tau,a_h^\tau) - \xi_{h}^{\tau}-V(s_{h+1}^\tau))|\cG_{h,1}^\tau]\\
    &=2(f(s_h^\tau,a_h^\tau)-f_V(s_h^\tau,a_h^\tau))\mathbb{E}[(f_V(s_h^\tau,a_h^\tau)-r_h^\tau(s_h^\tau,a_h^\tau) - \xi_{h}^{\tau}-V(s_{h+1}^\tau))|\cG_{h,1}^\tau]\\
    &=2(f(s_h^\tau,a_h^\tau)-f_V(s_h^\tau,a_h^\tau))(f_V(s_h^\tau,a_h^\tau)-r_h^\tau(s_h^\tau,a_h^\tau)-P_hV(s_h^\tau,a_h^\tau))\\
    &=0.
\end{align*}
In addition, conditioning on the good event $\cG(K,H,\delta)$, we have
\begin{equation*}
    \left|\chi_h^\tau(f)\right|\leq 2(H+1+ \sqrt{\gamma_\tau})|f(s_h^\tau,a_h^\tau)-f_V(s_h^\tau,a_h^\tau)|.
\end{equation*}
As $\chi_h^\tau(f)$ is a martingale difference sequence conditioned on the filtration $\cG_{h,1}^\tau$ , by Azuma-Hoeffding inequality, we have
$$\mathbb{P}\left[\left|\sum_{\tau\in[k-1]}\chi_h^\tau(f)\right|\geq\epsilon\right]\leq2\mathrm{exp}\left(-\frac{\epsilon^2}{8(H+1+\sqrt{\gamma_\tau})^2\|f-f_V\|^2_{\cZ_h^k}}\right).$$

Now we set
\begin{align*}
    \epsilon&=\sqrt{8(H+1+\sqrt{\gamma_\tau})^2\log\left(\frac{2 \mathcal{N}(\cF,1/T)}{\delta}\right)\|f-f_V\|_{\cZ_h^k}^2}\\
    &\leq 4(H+1+\sqrt{\gamma_\tau})\|f-f_V\|_{\cZ_h^k}\sqrt{\log(2/\delta)+\log\mathcal{N}(\cF,1/T)}.
\end{align*}

With union bound, for all $g\in\cC(\cF,1/T)$, with probability at least $1-\delta$ we have
$$\left|\sum_{(\tau)\in[k-1]}\xi_h^\tau(g)\right|\leq 4(H+ 1 + \sqrt{\gamma_\tau})\|f-f_V\|_{\cZ_h^k}\sqrt{\log(2/\delta)+\log\mathcal{N}(\cF,1/T)}.$$

Thus, for all $f\in\cF$, there exists $g\in\cC(\cF,1/T)$ such that $\|f-g\|_\infty\leq1/T$ and
\begin{align*}
    \left|\sum_{(\tau)\in[k-1]}\chi_h^\tau(f)\right|&\leq\left|\sum_{(\tau)\in[k-1]}\chi_h^\tau(g)\right|+2(H+1+\sqrt{\gamma_\tau})\\
    &\leq 4(H+1 +\sqrt{\gamma_\tau})\|g-f_V\|_{\cZ_h^k}\sqrt{\log\left(2/\delta\right)+\log\mathcal{N}\left(\cF,1/T\right)}+2(H+1+\sqrt{\gamma_\tau})\\
    &\leq 4(H+1 + \sqrt{\gamma_\tau})(\|f-f_V\|_{\cZ_h^k}+1)\sqrt{\log\left(2/\delta\right)+\log\mathcal{N}\left(\cF,1/T\right)}+2(H+1+\sqrt{\gamma_\tau}).\\
\end{align*}

For $V':\cS\rightarrow[0,H]$ such that $\|V-V'\|_\infty\leq1/T$, we have $\|f_{V'}-f_V\|_\infty\leq\|V'-V\|_\infty\leq1/T$.

For any $f \in \cF$, we have
\begin{align*}
    & \|f\|_{\widetilde{\cD}_{h,V'}^k}^2-\|f_{V'}\|_{\widetilde{\cD}_{h,V'}^k}^2\\ 
    =& \|f-f_{V'}\|_{\cZ_h^k}^2+2\sum_{(s_{h}^\tau,a_{h}^\tau)\in\cZ_h^k}(f(s_{h}^\tau,a_{h}^\tau)-f_{V'}(s_{h}^\tau,a_{h}^\tau))(f_{V'}(s_{h}^\tau,a_{h}^\tau)-r_h^\tau(s_h^\tau,a_h^\tau) - \xi_{h}^{\tau}-V'(s_{h+1}^\tau))\\
    \geq & \|f-f_{V'}\|_{\cZ_h^k}^2+2\sum_{(s_{h}^\tau,a_{h}^\tau)\in\cZ_h^k}(f(s_{h}^\tau,a_{h}^\tau)-f_{V}(s_{h}^\tau,a_{h}^\tau))(f_{V}(s_{h}^\tau,a_{h}^\tau)-r_h^\tau(s_h^\tau,a_h^\tau) - \xi_{h}^{\tau}-V(s_{h+1}^\tau))\\&-4(H+1+\sqrt{\gamma_k})\|V'-V\|_\infty|\cZ_h^k|\\
    \geq& \|f-f_{V'}\|_{\cZ_h^k}^2+\sum_{(\tau,h)\in[k-1]\times[H]}\chi_h^\tau(f)-4(H+ 1 + \sqrt{\gamma_k})\\
    \geq& \|f-f_{V'}\|_{\cZ_h^k}^2-4(H+ 1 + \sqrt{\gamma_k})(\|f-f_V\|_{\cZ_h^k}+1)\sqrt{\log\left(2/\delta\right)+\log\mathcal{N}\left(\cF,1/T\right)}-6(H+1 + \sqrt{\gamma_k})\\
    \geq& \|f-f_{V'}\|_{\cZ_h^k}^2-4(H+1 + \sqrt{\gamma_k})(\|f-f_{V'}\|_{\cZ_h^k}+2)\sqrt{\log\left(2/\delta\right)+\log\mathcal{N}\left(\cF,1/T\right)}-6(H+1 + \sqrt{\gamma_k}).\\
\end{align*}

In addition, using Assumption \ref{assumption:regularization}, we have the approximate triangle inequality for the perturbed regularizer:
\begin{align*}
    &\Tilde{R}(f)-\Tilde{R}(f_{V'})\\
    =&\sum_{i}^D[p_i(f)+\xi_i']^2-\sum_i^D[p_i(f_{V'})+\xi_i']^2\\
    =&R(f)-R(f_{V'})+2\sum_i^D \xi_i'(p_i(f)-p_i(f_{V'}))\\
    \geq& c R(f-f_{V'})-2R(f_{V'})-2\sum_i^D \sqrt{\gamma_k}p_i(f_{V'})\\
    \geq& c R(f-f_{V'})-2B-2\sqrt{\gamma_k}\sqrt{BD}.
\end{align*}

Summing the above two inequalities we have

$$\|f\|_{\widetilde{\cD}_{h,V'}^k}^2+\Tilde{R}(f)-\|f_{V'}\|_{\widetilde{\cD}_{h,V'}^k}^2-\Tilde{R}(f_{V'})\geq\|f-f_{V'}\|_{\cZ_h^k}^2+c R(f-f_{V'})-C,$$
where $C=4(H+1 + \sqrt{\gamma_k})(\|f-f_{V'}\|_{\cZ_h^k}+2)\sqrt{\log\left(2/\delta\right)+\log\mathcal{N}\left(\cF,1/T\right)}+6(H+1 + \sqrt{\gamma_k})+2B+2\sqrt{\gamma_k}\sqrt{BD}$.

As $\widetilde{f}_{h,V'}$ is the minimizer of $\|f\|_{\widetilde{\cD}_{h,V'}^k}^2+\Tilde{R}(f)$, we have
$$\|\widetilde{f}_{h,V'}-f_{V'}\|^2_{\cZ_h^k}+ c R(\widetilde{f}_{h,V'}-f_{V'})\leq c'\left[(H+1 + \sqrt{\gamma_k})\sqrt{\log\left(2/\delta\right)+\log\mathcal{N}\left(\cF,1/T\right)}+\sqrt{B+\sqrt{\gamma_kBD}}\right]^2.$$

To prove the above argument, we use the inequality that if we have $x^2+y\leq ax+b$ for positive $a,b,y$, then $x\leq a+\sqrt{b}$ and $x^2+y\leq (a+\sqrt{b})^2.$ In addition, we can remove $c$ by replacing $c'$ with $c'/\min\{1,c\}$ and then we get our final bound.

\end{proof}

\begin{lemma}[Confidence Region]\label{Lemma:GFA_r_plus_PV_in_f_khm}
Let $\cF_h^{k,m}=\{f\in\cF|\|f-\widetilde{f}_h^{k,m}\|_{\cZ_h^k}^2+R(f-\widetilde{f}_h^{k,m})\leq \beta(\cF,\delta)\}$, where 

\begin{equation}
    \beta(\cF,\delta)= c'\left[(H+1 + \sqrt{\gamma_k})\sqrt{\log\left(2/\delta\right)+\log\mathcal{N}\left(\cF,1/T\right)}+\sqrt{B+\sqrt{\gamma_kBD}}\right]^2.
\end{equation}

Conditioned on the event $\mathcal{G}(K, H, \delta)$, with probability at least $1-\delta$, for all $(k,h,m)\in[K]\times[H]\times[M]$, we have
$$r_h(\cdot,\cdot)+P_hV_{h+1}^k(\cdot,\cdot)\in\cF_h^{k,m}.$$
\end{lemma}

\begin{proof}
First note that for a fixed $(k,h,m)\in[K]\times[H]\times[M]$,
$$\cQ=\{\min\{f(\cdot,\cdot),H\}\mid f\in\cC(\cF,1/T)\}\cup\{0\}$$
is a $(1/T)$-cover of $Q_{h+1}^{k,m}(\cdot,\cdot)$. 
This implies $\cQ$ is also a $(1/T)$-cover of $Q_{h+1}^{k}(\cdot,\cdot)$. This further implies 
$$\cV=\{\max_{a\in\cA}q(\cdot,a) \mid q \in\cQ\}$$
is a $1/T$ cover of $V_{h+1}^{k}(\cdot)$ where we have $\log(|\cV|)=\log\mathcal{N}(\cF,1/T)$.

For the remaining part of the proof, we condition on $\bigcap_{V\in\cV}\cE_{h,V}(\delta/|\cV|TM)$, where $\cE_{h,V}(\delta)$ is the event defined in Lemma \ref{Lemma:GFA-Concentration}. By Lemma~\ref{Lemma:GFA-Concentration} and union bound, we have $\Pr\left[\bigcap_{V\in \cV} \cE_{h,V}(\delta/(8|\cV|MT)
\right]\ge 1-\delta/(8MT)$. 

Let $V\in\cV$ such that $\|V-V_{h+1}^k\|_\infty\leq 1/T$. By Lemma \ref{Lemma:GFA-Concentration} we have
\begin{align*}
    \left\|\widetilde{f}_h^{k,m}(\cdot,\cdot)-r_h(\cdot,\cdot)-P_h V_{h+1}^k(\cdot, \cdot)\right\|^2_{\cZ_h^k}+R(\widetilde{f}_h^{k,m}(\cdot,\cdot)-r_h(\cdot,\cdot)-P_h V_{h+1}^k(\cdot, \cdot))&\\\leq c'\left[(H+1+ \sqrt{\gamma_k})\sqrt{\log\left(1/\delta\right)+\log\mathcal{N}\left(\cF,1/T\right)}\right]^2,
\end{align*}

where $c'$ is some absolute constant. By union bound, for all $(k,h,m) \in [K]\times[H]\times[M]$ we have $r_h(\cdot, \cdot)+P_h V_{h+1}^k( \cdot, \cdot) \in \cF_{h}^{k,m}$ with probability $1-\delta$.
\end{proof}

The last lemma guarantees that $r_h(\cdot,\cdot) + P_hV_{h+1}^k(\cdot,\cdot)$ lies in the confidence region $\cF_h^{k,m}$ with high probability. Note that the confidence region $\cF_h^{k,m}$ is centered at $\widetilde{f}_h^{k,m}$, which is the solution to the perturbed regression problem defined in ~\eqref{eq:perturbed-regression-problem}. For the unperturbed regression problem and its solution as center of the confidence region, we get the following lemma as a direct consequence of Lemma~\ref{Lemma:GFA_r_plus_PV_in_f_khm}.
\begin{lemma}
\label{Lemma:GFA_r_plus_PV_in_f_khm (unperturbed)}
Let $\cF_h^{k}=\{f\in\cF|\|f-\Hat{f}_h^{k}\|_{\cZ_h^k}^2+R(f-\Hat{f}_h^{k})\leq \beta'(\cF,\delta)\}$, where 
\begin{equation}
    \beta'(\cF,\delta)\geq c'\left[(H+1)\sqrt{\log\left(2/\delta\right)+\log\mathcal{N}\left(\cF,1/T\right)}+\sqrt{B}\right]^2.
\end{equation}
With probability at least $1-\delta$, for all $(k,h,m)\in[K]\times[H]\times[M]$, we have
\begin{equation*}
    r_h(\cdot,\cdot)+P_hV_{h+1}^k(\cdot,\cdot)\in\cF_h^{k}.
\end{equation*}
\end{lemma}

\begin{proof}
This is a direct implication of Lemma \ref{Lemma:GFA_r_plus_PV_in_f_khm} with zero perturbance.
\end{proof}

\subsection{Optimism}
In this section, we will show that $\{Q_h^k\}_{(h,k)\in[H]\times[K]}$ is optimistic with high probability. Formally, we have the following lemma.

\begin{lemma}\label{lemma:optimism-GFA}
Set $M=\ln(\frac{T|\cS||\cA|}{\delta})/\ln(\frac{1}{1-v})$ in Algorithm \ref{Algorithm-GFA}. Conditioned on the event $\mathcal{G}(K, H, \delta)$, with probability at least $1-\delta$, for all $s\in\mathcal{S}$, $a\in\mathcal{A}$, $h\in[H]$, $k\in[K]$, we have
$$Q_h^*(s,a)\leq Q_h^k(s,a).$$
\end{lemma}

\begin{proof}
For timestep $H+1$, we have $Q_{H+1}^k=Q_{H+1}^*=0$. By Lemma \ref{Lemma:GFA_r_plus_PV_in_f_khm (unperturbed)}, there exists $\beta'(\cF,\delta)$ such that with probability at least $1-\delta$, for all $(k,h)\in[K]\times[H]$, we have
$$f_h^k(\cdot,\cdot):=r_h(\cdot,\cdot)+P_hV_{h+1}^k(\cdot,\cdot)\in\cF_h^{k},$$
where $\cF_h^k=\{f\in\cF\,|\,\|f-\Hat{f}_h^{k}\|_{\cZ_h^k}^2+R(f-\Hat{f}_h^{k})\leq \beta'(\cF,\delta)\}.$ 

Using notations introduced in Definition~\ref{def:anticoncentration width}, let $g_{h,\sigma}^k$ be a function such that  $\Tilde{f}_h^{k,m}(s,a)\geq\Hat{f}(s,a)+g_{h,\sigma}^k(s,a)$ holds with probability at least $v$. We set $M=\ln(\frac{T|\cS||\cA|}{\delta})/\ln(\frac{1}{1-v})$ and then $\Tilde{f}_h^{k,m}(s,a)\geq\Hat{f}(s,a)+g_{h,\sigma}^k(s,a)$ with probability at least
$$1-(1-v)^M=1-\frac{\delta}{T|\cS||\cA|},$$
for any $(k,h)\in[K]\times[H]$ and $(s,a)\in \cS\times\cA$. By union bound, we have $\Tilde{f}_h^{k,m}(s,a)\geq\Hat{f}(s,a)+g_{h,\sigma}^k(s,a)$ for all $(k,h)\in[K]\times[H]$ and $(s,a)\in \cS\times\cA$ with probability at least $1-\delta$ and we have
\begin{align*}
    \Tilde{f}_h^k(s,a)&=\max_{m\in[M]}\Tilde{f}_h^{k,m}(s,a)\\
    &\geq\Hat{f}_h^k(s,a)+g^k_{h,\sigma}(s,a)\\
    &\geq\Hat{f}_h^k(s,a)+w(\cF_h^k)\\
    &\geq f_h^k(s,a),
\end{align*}
where the second inequality is from Assumption \ref{assumption:anticoncentration} and the choise of $\sigma$ as discussed in Appendix~\ref{sec:noise}. The last inequality follows from the definition of the width function and the previous observation that $f_h^k(\cdot,\cdot)\in\cF_h^{k}$ with probability at least $1-\delta$.
Now we induct on $h$ from $h=H$ to 1.
\begin{align*}
    Q_h^*(s,a)&=\min \{ r_h(s,a)+P_h V_{h+1}^*(s,a),H\}\\
    &=\min \{f_h^k(s,a)+P_h(V_{h+1}^*-V_{h+1}^k)(s,a),H\}\\
    &\leq \min\{\Tilde{f}_h^k(s,a)+P_h(V_{h+1}^*-V_{h+1}^k)(s,a),H\}\\
    &\leq \min\{\Tilde{f}_h^k(s,a),H\}\\
    &= Q_h^k(s,a).
\end{align*}
Thus,
\begin{align*}
    V_h^*(s)=\max_a Q_h^*(s,a)\leq \max_a Q_h^k(s,a)=V_h^k(s).
\end{align*}
where the second inequality is from $V_{h+1}^*\leq V_{h+1}^k$, which is implied by induction.
\end{proof}

\subsection{Regret Bound}
We are now ready to provide the regret bound for Algorithm~\ref{Algorithm-GFA}. The next lemma upper bounds the regret of the algorithm by the sum of the width functions.

\begin{lemma}[Regret decomposition]
Denote $b_h^k(s,a)=w(\cF_h^k,s,a)$.  Conditioned on the event $\mathcal{G}(K, H, \delta)$, with probability at least $1-\delta$, we have

$$\mathrm{Regret}(K)\leq \sum_{k=1}^K\sum_{h=1}^{H}b_h^k(s_h^k,a_h^k)+\sum_{k=1}^K\sum_{h=1}^H\zeta_h^k,$$

where $\zeta_h^k=P(s_h^k,a_h^k)(V_{h+1}^k-V_{h+1}^{\pi^k})-(V_{h+1}^k(s_{h+1}^k)-V_{h+1}^{\pi^k}(s_{h+1}^k))$ is a martingale difference sequence with respect to the filtration $\cG_{h,2}^k$.

\end{lemma}

\begin{proof}

    We condition on the good events in Lemma \ref{Lemma:GFA_r_plus_PV_in_f_khm}. For all $(k,h,m)\in[K]\times[H]\times[M]$, we have
    $$\left\|r_h(\cdot,\cdot)+P_hV_{h+1}^k(\cdot,\cdot)-\widetilde{f}_h^{k,m}\right\|_{\cZ_h^k}^2+R(r_h(\cdot,\cdot)+P_hV_{h+1}^k(\cdot,\cdot)-\widetilde{f}_h^{k,m})\leq \beta(\cF,\delta).$$
    
    Recall that $\cF_h^k=\{f\,|\,\left\|r_h(\cdot,\cdot)+P_hV_{h+1}^k(\cdot,\cdot)-f\right\|_{\cZ_h^k}^2+R(r_h(\cdot,\cdot)+P_hV_{h+1}^k(\cdot,\cdot)-\widetilde{f}_h^{k,m})\leq \beta(\cF,\delta)\}$ is the confidence region. Then for $(k,h,m)\in[K]\times[H]\times[M]$, $\widetilde{f}_h^{k,m}\in \cF_h^k$. Defining $b_h^k(s,a)=w(\cF_h^k,s,a)$, for all $(k,h,m)\in[K]\times[H]\times[M]$ we have, $$b_h^k(s,a)\geq \left|r(s,a)+P(s,a)V_{h+1}^k-\widetilde{f}_h^{k,m}(s,a)\right|.$$ As $Q_h^k(s,a)=\min\{\max_{m\in[M]}\{\widetilde{f}_h^{k,m}(\cdot,\cdot)\},H-h+1\}$, we have $$b_h^k(s,a)\geq\left|r(s,a)+P(s,a)V_{h+1}^k-Q_h^k(s,a)\right|.$$
    
By Lemma \ref{lemma:optimism-GFA} and standard telescoping argument, we have
\begin{align*}
    \mathrm{Regret}(K)&\leq\sum_{k=1}^K V_1^k(s_1^k)-V_1^{\pi_k}(s_1^k)\\
    &= \sum_{k=1}^K Q_1^k(s_1^k,a_1^k)-Q_1^{\pi^k}(s_1^k,a_1^k)\\
    &= \sum_{k=1}^K Q_1^k(s_1^k,a_1^k)-(r(s_1^k,a_1^k)+P(s_1^k,a_1^k)V_2^k)+(r(s_1^k,a_1^k)+P(s_1^k,a_1^k)V_2^k)-Q_1^{\pi^k}(s_1^k,a_1^k)\\
    &\leq \sum_{k=1}^K b_1^k(s_1^k,a_1^k)+P(s_1^k,a_1^k)(V_2^k-V_2^{\pi^k})\\
    &= \sum_{k=1}^K b_1^k(s_1^k,a_1^k)+(V_2^k(s_2^k)-V_2^{\pi^k}(s_2^k))+\zeta_1^k\\
    &\leq \sum_{k=1}^K \sum_{h=1}^H b_h^k(s_h^k,a_h^k) +\sum_{k=1}^K\sum_{h=1}^H\zeta_h^k.
\end{align*}
\end{proof}

\begin{lemma}[Time inhomogeneous version of Lemma 10 in \citep{wang2020reinforcement}]
Let $\cF'$ be a subset of function class $\cF$, consisting of all $f\in \cF$ such that
\begin{equation*}
    \|f-\upsilon\|^2_{\cZ}+R(f-\upsilon) \leq \beta(\cF,\delta),
\end{equation*}
where $v=r + PV$ as in Assumption \ref{assumption:bounded_function_class} and $\beta(\cF,\delta)$ as defined in Lemma \ref{Lemma:GFA_r_plus_PV_in_f_khm}. With probability at least $1-\delta$, we have
$$\sum_{k=1}^K \sum_{h=1}^H b_h^k(s_h^k,a_h^k)\leq H+4H^3\mathrm{dim_\mathcal{E}}(\cF',1/T)+H\sqrt{c\mathrm{dim_\mathcal{E}}(\cF',1/T)K\beta(\cF,\delta)},$$
for some absolute constant $c>0$.
\label{eluder lemma}
\end{lemma}

\begin{proof}
Define $$\cF^{\prime k}_h=\{f\in\cF^{\prime}\,|\,\|f-\Hat{f}_h^{k}\|_{\cZ_h^k}^2\leq \beta(\cF,\delta)\}=\cF^{\prime}\bigcap\{f\in\cF\,|\,\|f-\Hat{f}_h^{k}\|_{\cZ_h^k}^2\leq \beta(\cF,\delta)\}.$$ As $\cF_h^k\subseteq\cF^{\prime}$ and $\cF_h^k\subseteq\bigcap\{f\in\cF\,|\,\|f-\Hat{f}_h^{k}\|_{\cZ_h^k}^2\leq \beta(\cF,\delta)\}$, we have $\cF_h^k\subseteq\cF_h^{\prime k}$ and $w(\cF_h^k,s,a)\leq w(\cF_h^{\prime k},s,a)$ for all $s,a$. By Assumption \ref{assumption:bounded_function_class}, $\cF'$ has bounded eluder dimension.

Similar to Lemma 10 in \citep{wang2020reinforcement}, we have for any $h$,
$$\sum_{k=1}^K b_h^k(s_h^k,a_h^k)\leq\sum_{k=1}^K w(\cF'^k_h,s,a) \leq1+4H^2\mathrm{dim_\mathcal{E}}(\cF',1/T)+\sqrt{c\mathrm{dim_\mathcal{E}}(\cF',1/T)K\beta(\cF,\delta)}.$$
Summing over all timestep $h$ and we have the bound in the lemma.

\end{proof}

\begin{theorem} \label{thm:formal GFA}

Under all the assumptions, with probability at least $1-\delta$, Algorithm \ref{Algorithm-GFA} achieves a regret bound of 
$$\mathrm{Regret}(K)\leq4H^3\mathrm{dim_\mathcal{E}}(\cF,1/T)+\sqrt{\mathrm{dim_\mathcal{E}}(\cF,1/T)\beta(\cF,\delta)HT},$$
where
$$\beta(\cF,\delta)= c'\left[(H+1 + \sigma)\sqrt{\log\left(2/\delta\right)+\log\mathcal{N}\left(\cF,1/T\right)}+\sqrt{B+\sigma\sqrt{BD}}\right]^2,$$
for some constant $c'$.

\end{theorem}

\begin{proof}
By Assumption \ref{assumption:bounded_function_class}, we can consider $\cF'\subseteq\cF$ as the whole function class in the analysis because it includes all the $\cF_{h}^k,\forall h,k$. By Azuma-Hoeffding inequality and Lemma \ref{eluder lemma}, With probability at least $1-\delta$, we have
\begin{align*}
    \mathrm{Regret}(K)&\leq\sum_{k=1}^K\sum_{h=1}^{H}b_h^k(s_h^k,a_h^k)+\sum_{k=1}^K\sum_{h=1}^H\zeta_h^k\\
    &\leq c'\left(H+4H^3\mathrm{dim_\mathcal{E}}(\cF,1/T)+H\sqrt{c\mathrm{dim_\mathcal{E}}(\cF,1/T)K\beta(\cF,\delta)}+H\sqrt{KH\log\left(1/\delta\right)}\right),
\end{align*}
for some constant $c'$.
We plug in the definition of $\beta(\cF,\delta)$ and $\sqrt{\gamma_k}=\Tilde{O}(\sigma)$, then we get the final bound.
\end{proof}

\begin{remark}
For linear MDP, as shown in Section \ref{subsection:assumption-linear-function-class}, we have $$\sigma=2\sqrt{\beta'(\cF,\delta)}=c'\left[(H+1)\sqrt{\log\left(2/\delta\right)+\log\mathcal{N}\left(\cF,1/T\right)}+\sqrt{B}\right]^2,$$$B=2Hd$ and $D=d$. In addition, we have $dim_\mathcal{E}(\cF,1/T)=\Tilde{O}(d)$ \citep{russo2013eluder} and $\log\mathcal{N}\left(\cF,1/T\right)=\Tilde{O}(d)$. As a result, our bound implies a $\Tilde{O}(\sqrt{H^3d^3T})$ regret bound for linear MDP.
\end{remark}

\section{GFA With Model Misspecification}

\begin{assumption} \label{assumption:GFA misspecification} (Assumption 3 in \citep{wang2020reinforcement})
For function class $\cF$, there exists a real number $\zeta$, such that for any $V:\cS\rightarrow[0,H]$, there exists $g_V\in\cF$ which satisfies
$$\max_{(s,a)\in\cS\times\cA}\left|g_V(s,a)-r(s,a)-\sum_{s'\in\cS}P(s'|s,a)V(s')\right|\leq\zeta.$$
In addition, we assume $g_V$ satisfies Assumption \ref{assumption:regularization}, i.e. $R(g_V)\leq B.$
\end{assumption}

\begin{lemma}\label{Lemma:GFA-Concentration_misspecified} Consider a fixed $k \in [K]$ and a fixed $h \in [H]$. Let $\cZ_h^k=\{(s_{h}^\tau,a_{h}^\tau)\}_{\tau\in[k-1]}$ and $\Tilde{\mathcal{D}}_{h,V}^{k}=\{(s_{h}^\tau,a_{h}^\tau, r_{h}^\tau+\xi_{h}^{\tau}+V(s_{h+1}^\tau))\}_{\tau\in[k-1]}$. Define
$\Tilde{f}_{h,V}^{k}= \argmin_{f\in\mathcal{F}}\|f\|_{\Tilde{\mathcal{D}}_{h,V}^{k}}^2+\Tilde{R}(f).$ Conditioned on the good event $\cG(K,H,\delta)$, with probability at least $1-\delta$, for a fixed $V:\mathcal{S}\rightarrow[0,H]$ and any $V':\mathcal{S}\rightarrow[0,H]$ with $\|V'-V\|_\infty\leq1/T$, we have

\begin{align*}
&\left\|\widetilde{f}_{h,V'}(\cdot,\cdot)-r_h(\cdot,\cdot)-P_hV'( \cdot,\cdot)\right\|^2_{\mathcal{Z}_h^k}+ R(\widetilde{f}_{h,V'}(\cdot,\cdot)-r_h(\cdot,\cdot)-P_hV'( \cdot,\cdot))\\\leq& c'\left[(H+1 + \sqrt{\gamma_k})\sqrt{\log\left(2/\delta\right)+\log\mathcal{N}\left(\cF,1/T\right)}+\sqrt{B+\sqrt{\gamma_kBD}+\zeta K(H+\sqrt{\gamma_k})}\right]^2,
\end{align*}
for some constant $c'$.
\end{lemma}

\begin{proof}
Recall that for notational simplicity, we denote $[\mathbb{P}_h V_{h+1}](s,a) = \mathbb{E}_{s' \sim \mathbb{P}_h(\cdot \,|\, s, a)}V_{h+1}(s')$.
Now consider a fixed $V:\cS\rightarrow [0,H]$, and define
\begin{equation}
    f_V(\cdot, \cdot) = r_h(\cdot,\cdot) +  P_hV( \cdot, \cdot).
\end{equation}
By Assumption \ref{assumption:GFA misspecification}, there exists $g_V\in\cF$ such that
$$\max_{(s,a)\in\cS\times\cA}\left|g_V(s,a)-f_V(s,a)\right|\leq\zeta.$$

For any $f \in \cF$, consider $$\chi_h^\tau=2(f(s_h^\tau,a_h^\tau)-f_V(s_h^\tau,a_h^\tau))(f_V(s_h^\tau,a_h^\tau)-r_h^\tau(s_h^\tau,a_h^\tau) - \xi_{h}^{\tau} -V(s_{h+1}^\tau)).$$
First we show that $\chi_h^\tau(f)$ is a martingale difference sequence with respect to the filtration $\cG_{h,1}^\tau$.
\begin{align*}
    \mathbb{E}[\chi_h^\tau(f)|\cG_{h,1}^\tau]&=\mathbb{E}[2(f(s_h^\tau,a_h^\tau)-f_V(s_h^\tau,a_h^\tau))(f_V(s_h^\tau,a_h^\tau)-r_h^\tau(s_h^\tau,a_h^\tau) - \xi_{h}^{\tau}-V(s_{h+1}^\tau))|\cG_{h,1}^\tau]\\
    &=2(f(s_h^\tau,a_h^\tau)-f_V(s_h^\tau,a_h^\tau))\mathbb{E}[(f_V(s_h^\tau,a_h^\tau)-r_h^\tau(s_h^\tau,a_h^\tau) - \xi_{h}^{\tau}-V(s_{h+1}^\tau))|\cG_{h,1}^\tau]\\
    &=2(f(s_h^\tau,a_h^\tau)-f_V(s_h^\tau,a_h^\tau))(f_V(s_h^\tau,a_h^\tau)-r_h^\tau(s_h^\tau,a_h^\tau)-P_hV(s_h^\tau,a_h^\tau))\\
    &=0.
\end{align*}

In addition, conditioning on good events $\cG(K,H,\delta)$, we have
$$|\chi_h^\tau(f)|\leq 2(H+1+ \sqrt{\gamma_\tau})|f(s_h^\tau,a_h^\tau)-f_V(s_h^\tau,a_h^\tau)|.$$

As $\chi_h^\tau(f)$ is a martingale difference sequence conditioned on the filtration $\cG_{h,1}^\tau$ , by Azuma-Hoeffding inequality, we have
$$\mathbb{P}\left[\left|\sum_{\tau\in[k-1]}\chi_h^\tau(f)\right|\geq\epsilon\right]\leq2\mathrm{exp}\left(-\frac{\epsilon^2}{8(H+1+\sqrt{\gamma_\tau})^2\|f-f_V\|^2_{\cZ_h^k}}\right).$$

Now we set
\begin{align*}
    \epsilon&=\sqrt{8(H+1+\sqrt{\gamma_\tau})^2\log\left(\frac{2 N(\cF,1/T)}{\delta}\right)\|f-f_V\|_{\cZ_h^k}^2}\\
    &\leq 4(H+1+\sqrt{\gamma_\tau})\|f-f_V\|_{\cZ_h^k}\sqrt{\log(2/\delta)+\log\mathcal{N}(\cF,1/T)}.
\end{align*}

With union bound, for all $g\in\cC(\cF,1/T)$, with probability at least $1-\delta$ we have
$$\left|\sum_{(\tau)\in[k-1]}\xi_h^\tau(g)\right|\leq 4(H+ 1 + \sqrt{\gamma_\tau})\|f-f_V\|_{\cZ_h^k}\sqrt{\log(2/\delta)+\log\mathcal{N}(\cF,1/T)}.$$

Thus, for all $f\in\cF$, there exists $g\in\cC(\cF,1/T)$ such that $\|f-g\|_\infty\leq1/T$ and ,
\begin{align*}
    \left|\sum_{(\tau)\in[k-1]}\chi_h^\tau(f)\right|&\leq\left|\sum_{(\tau)\in[k-1]}\chi_h^\tau(g)\right|+2(H+1+\sqrt{\gamma_\tau})\\
    &\leq 4(H+1 +\sqrt{\gamma_\tau})\|g-f_V\|_{\cZ_h^k}\sqrt{\log\left(2/\delta\right)+\log\mathcal{N}\left(\cF,1/T\right)}+2(H+1+\sqrt{\gamma_\tau})\\
    &\leq 4(H+1 + \sqrt{\gamma_\tau})(\|f-f_V\|_{\cZ_h^k}+1)\sqrt{\log\left(2/\delta\right)+\log\mathcal{N}\left(\cF,1/T\right)}+2(H+1+\sqrt{\gamma_\tau})\\
\end{align*}

For $V':\cS\rightarrow[0,H]$ such that $\|V-V'\|_\infty\leq1/T$, we have $\|f_{V'}-f_V\|_\infty\leq\|V'-V\|_\infty\leq1/T$.

For any $f \in \cF$, we have
\begin{align*}
    & \|f\|_{\widetilde{\cD}_{h,V'}^k}^2-\|f_{V'}\|_{\widetilde{\cD}_{h,V'}^k}^2\\ 
    =& \|f-f_{V'}\|_{\cZ_h^k}^2+2\sum_{(s_{h}^\tau,a_{h}^\tau)\in\cZ_h^k}(f(s_{h}^\tau,a_{h}^\tau)-f_{V'}(s_{h}^\tau,a_{h}^\tau))(f_{V'}(s_{h}^\tau,a_{h}^\tau)-r_h^\tau(s_h^\tau,a_h^\tau) - \xi_{h}^{\tau}-V'(s_{h+1}^\tau))\\
    \geq & \|f-f_{V'}\|_{\cZ_h^k}^2+2\sum_{(s_{h}^\tau,a_{h}^\tau)\in\cZ_h^k}(f(s_{h}^\tau,a_{h}^\tau)-f_{V}(s_{h}^\tau,a_{h}^\tau))(f_{V}(s_{h}^\tau,a_{h}^\tau)-r_h^\tau(s_h^\tau,a_h^\tau) - \xi_{h}^{\tau}-V(s_{h+1}^\tau))\\&-4(H+1+\sqrt{\gamma_k})\|V'-V\|_\infty|\cZ_h^k|\\
    \geq& \|f-f_{V'}\|_{\cZ_h^k}^2+\sum_{(\tau,h)\in[k-1]\times[H]}\chi_h^\tau(f)-4(H+ 1 + \sqrt{\gamma_k})\\
    \geq& \|f-f_{V'}\|_{\cZ_h^k}^2-4(H+ 1 + \sqrt{\gamma_k})(\|f-f_V\|_{\cZ_h^k}+1)\sqrt{\log\left(2/\delta\right)+\log\mathcal{N}\left(\cF,1/T\right)}-6(H+1 + \sqrt{\gamma_k})\\
    \geq& \|f-f_{V'}\|_{\cZ_h^k}^2-4(H+1 + \sqrt{\gamma_k})(\|f-f_{V'}\|_{\cZ_h^k}+2)\sqrt{\log\left(2/\delta\right)+\log\mathcal{N}\left(\cF,1/T\right)}-6(H+1 + \sqrt{\gamma_k}).\\
\end{align*}

In addition, by Assumption \ref{assumption:regularization}, we have
\begin{align*}
    &\Tilde{R}(f)-\Tilde{R}(f_{V'})\\
    =&\sum_{i}[p_i(f)-\xi_i']^2-\sum_i[p_i(f_{V'})-\xi_i']^2\\
    =&R(f)-R(f_{V'})-2\sum_i \xi_i'(p_i(f)-p_i(f_{V'}))\\
    \geq& c R(f-f_{V'})-2R(f_{V'})-2\sum_i \sqrt{\gamma_k}p_i(f_{V'})\\
    \geq& c R(f-f_{V'})-2B-2\sqrt{\gamma_k}\sqrt{BD}.
\end{align*}

Summing the above two inequalities we have

$$\|f\|_{\widetilde{\cD}_{h,V'}^k}^2+\Tilde{R}(f)-\|f_{V'}\|_{\widetilde{\cD}_{h,V'}^k}^2-\Tilde{R}(f_{V'})\geq\|f-f_{V'}\|_{\cZ_h^k}^2+c R(f-f_{V'})-C,$$
where $C=4(H+1 + \sqrt{\gamma_k})(\|f-f_{V'}\|_{\cZ_h^k}+2)\sqrt{\log\left(2/\delta\right)+\log\mathcal{N}\left(\cF,1/T\right)}+6(H+1 + \sqrt{\gamma_k})+2B+2\sqrt{\gamma_k}\sqrt{DB}$.

Now we try to replace the $f_{V'}$ in the RHS with $g_V'$.
\begin{align*}
    &\|f_{V'}\|_{\widetilde{\cD}_{h,V'}^k}^2-\|g_{V'}\|_{\widetilde{\cD}_{h,V'}^k}^2\\
    =&\sum_{\tau\in[k-1]}(f_{V'}(s_h^\tau,a_h^\tau)-(r_h^\tau+\xi_h^\tau+V(s_{h+1}^\tau)))^2-\sum_{\tau\in[k-1]}(g_{V'}(s_h^\tau,a_h^\tau)-(r_h^\tau+\xi_h^\tau+V(s_{h+1}^\tau)))^2\\
    =&\sum_{\tau\in[k-1]}(f_{V'}(s_h^\tau,a_h^\tau)-g_{V'}(s_h^\tau,a_h^\tau))(f_{V'}(s_h^\tau,a_h^\tau)+g_{V'}(s_h^\tau,a_h^\tau)-2(r_h^\tau+\xi_h^\tau+V(s_{h+1}^\tau)))\\
    \geq&-\zeta K(4H+2\sqrt{\gamma_k}).
\end{align*}
By the boundedness of the regularizer (Assumption \ref{assumption:regularization}), we have
\begin{align*}
    \|f_{V'}\|_{\widetilde{\cD}_{h,V'}^k}^2+\Tilde{R}(f_{V'})-\|g_{V'}\|_{\widetilde{\cD}_{h,V'}^k}^2-\Tilde{R}(g_{V'})\geq-\zeta K(4H+2\sqrt{\gamma_k})-B.
\end{align*}
Thus we have
\begin{align*}
    \|f\|_{\widetilde{\cD}_{h,V'}^k}^2+\Tilde{R}(f)-\|g_{V'}\|_{\widetilde{\cD}_{h,V'}^k}^2-\Tilde{R}(g_{V'})&\geq\|f\|_{\widetilde{\cD}_{h,V'}^k}^2+\Tilde{R}(f)-\|f_{V'}\|_{\widetilde{\cD}_{h,V'}^k}^2-\Tilde{R}(f_{V'})-\zeta K(4H+2\sqrt{\gamma_k})-B\\
    &\geq\|f-f_{V'}\|_{\cZ_h^k}^2+c R(f-f_{V'})-C-\zeta K(4H+2\sqrt{\gamma_k})-B.
\end{align*}

As $\widetilde{f}_{h,V'}$ is the minimizer of $\|f\|_{\widetilde{\cD}_{h,V'}^k}^2+\Tilde{R}(f)$ for $f\in\cF$ and note that $g_{V'}\in\cF$, we have
\begin{align*}
    &\|\widetilde{f}_{h,V'}-f_{V'}\|^2_{\cZ_h^k}+ c R(\widetilde{f}_{h,V'}-f_{V'})\\\leq& c'\left[(H+1 + \sqrt{\gamma_k})\sqrt{\log\left(2/\delta\right)+\log\mathcal{N}\left(\cF,1/T\right)}+\sqrt{B+\sqrt{\gamma_kBD}+\zeta K(H+\sqrt{\gamma_k})}\right]^2.
\end{align*}

To prove the above argument, we use the inequality that if we have $x^2+y\leq ax+b$ for positive $a,b,y$, then $x\leq a+\sqrt{b}$ and $x^2+y\leq (a+\sqrt{b})^2.$ In addition, we can remove $c$ by replacing $c'$ with $c'/\min\{1,c\}$ and then we get the final bound.

\end{proof}

\begin{lemma} (Misspecified Confidence Region) \label{Lemma:misspecified confidence region}
Let $\cF_h^{k,m}=\{f\in\cF|\|f-\widetilde{f}_h^{k,m}\|_{\cZ_h^k}^2+R(f-\widetilde{f}_h^{k,m})\leq \beta(\cF,\delta)\}$, where 

\begin{equation}
    \beta(\cF,\delta)= c'\left[(H+1 + \sqrt{\gamma_k})\sqrt{\log\left(2/\delta\right)+\log\mathcal{N}\left(\cF,1/T\right)}+\sqrt{B+\sqrt{\gamma_kBD}+\zeta K(H+\sqrt{\gamma_k})}\right]^2.
\end{equation}

Conditioned on the event $\mathcal{G}(K, H, \delta)$, with probability at least $1-\delta$, for all $(k,h,m)\in[K]\times[H]\times[M]$, we have
$$r_h(\cdot,\cdot)+P_hV_{h+1}^k(\cdot,\cdot)\in\cF_h^{k,m}.$$

\end{lemma}

\begin{proof}
With Lemma \ref{Lemma:GFA-Concentration_misspecified}, the proof is same as Lemma \ref{Lemma:GFA_r_plus_PV_in_f_khm}.
\end{proof}

\begin{theorem} \label{thm:formal GFA misspecified}

Under all the assumptions, with probability at least $1-\delta$, Algorithm \ref{Algorithm-GFA} achieves a regret bound of 
$$\mathrm{Regret}(K)\leq4H^3\mathrm{dim_\mathcal{E}}(\cF,1/T)+\sqrt{\mathrm{dim_\mathcal{E}}(\cF,1/T)\beta(\cF,\delta)HT},$$
where
$$\beta(\cF,\delta)= c'\left[(H+1 + \sigma)\sqrt{\log\left(2/\delta\right)+\log\mathcal{N}\left(\cF,1/T\right)}+\sqrt{B+\sigma\sqrt{BD}+\zeta K(H+\sigma)}\right]^2,$$
for some constant $c'$.

\end{theorem}

\begin{proof}
With Lemma \ref{Lemma:misspecified confidence region}, the proof is the same as Theorem \ref{thm:formal GFA}.
\end{proof}
\section{LSVI-PHE with linear function approximation}\label{App-definitions}
In this section, we prove Theorem~\ref{thm:theorem-linear MDP}. Our analysis specilized to linear MDP setting is simpler and may provide additional insights. In addition, compared to GFA setting, we improve the bound for $M$ and it no longer depends on $|\cS|$ or $|\cA|$. We first introduce the notation and few definitions that are used throughout this section. Upon presenting lemmas and their proofs, finally we combine the lemmas to prove Theorem~\ref{thm:theorem-linear MDP}.

\begin{definition}[Model prediction error]
For all $(k,h) \in [K]\times[H]$, we define the model prediction error associated with the reward $r_h^k$,
\begin{equation*}
    l^{k}_{h}(s,a)  =  r^k_h (s,a) + \mathbb{P}_{h}V^{k}_{h+1}(s,a) - Q^{k}_h(s,a).
\end{equation*}
This depicts the prediction error using $V_{h+1}^k$ instead  of $V_{h+1}^{\pi^k}$ in the Bellman equations \eqref{eq:bellman-eq}.
\end{definition}

\begin{definition}[Unperturbed estimated parameter]\label{Definition:unperturbed-regressor}
For all $(k,h) \in [K]\times[H]$, we define the unperturbed estimated parameter as
\begin{equation*}
    \Hat{\theta}^{k}_h=(\Lambda_h^k)^{-1}
    \left(\sum_{\tau=1}^{k-1}[r_h^\tau+V_{h+1}^k(s_{h+1}^\tau)]\phi(s_h^\tau,a_h^\tau)\right).
\end{equation*}

Moreover, we denote the difference between the perturbed estimated parameter $\Tilde{\theta}^{k,j}_h$ and the unperturbed estimated parameter $\hat{\theta_h^k}$ as 
\begin{equation*}
    \zeta_h^{k,j} = \Tilde{\theta}^{k,j}_h - \hat{\theta_h^k}.
\end{equation*}
\end{definition}

\subsection{Concentration}

Our first lemma characterizes the difference between the perturbed estimated parameter  $\Tilde{\theta}^{k,j}_h$ and the unperturbed estimated parameter  $\Hat{\theta}^{k}_h$.
\begin{proposition}[restatement of Proposition~\ref{Proposition:noise}]
In step 9 of Algorithm~\ref{Algorithm-linear MDP reward perturbation}, conditioned on all the randomness except $\{\epsilon^{k,i,j}_{h}\}_{(i,j) \in [k-1]\times[M]}$ and $\{\xi^{k,j}_{h}\}_{j \in [M]}$, the estimated parameter $\Tilde{\theta}^{k,j}_h$ satisfies
\begin{equation*}
    \zeta_h^{k,j} = \Tilde{\theta}^{k,j}_h - \Hat{\theta}^{k}_h\sim N(0,\sigma^2(\Lambda_h^k)^{-1}),
\end{equation*}
where $\Hat{\theta}^{k}_h=(\Lambda_h^k)^{-1}(\sum_{\tau=1}^{k-1}[r_h^\tau+V_{h+1}^k(s_{h+1}^\tau)]\phi(s_h^\tau,a_h^\tau))$ is the unperturbed estimated parameter from Definiton~\ref{Definition:unperturbed-regressor}.
\end{proposition}
\begin{proof}

From Algorithm~\ref{Algorithm-linear MDP reward perturbation}, note that
\begin{align*}
    \Tilde{\theta}^{k,j}_h &= (\Lambda^{k}_h)^{-1} (\rho_h^k +  \xi^{k,j}_{h})\\
    &= (\Lambda^{k}_h)^{-1} \left(\sum_{\tau=1}^{k-1}\left( [ r^\tau_{h} + V_{h+1}^{k}(s_{h+1}^\tau)+\epsilon^{k,\tau,j}_{h}] \phi(s_h^\tau,a_h^\tau) \right) +  \xi^{k,j}_{h}\right)\\
    &=(\Lambda^{k}_h)^{-1} \left(\sum_{\tau=1}^{k-1} [ r^\tau_{h} + V_{h+1}^{k}(s_{h+1}^\tau)] \phi(s_h^\tau,a_h^\tau)  \right) + (\Lambda^{k}_h)^{-1}\left(\sum_{\tau=1}^{k-1}\epsilon^{k,\tau,j}_{h} \phi(s_h^\tau,a_h^\tau) + \xi^{k,j}_{h}\right)\\
    &= \Hat{\theta}^{k}_h + (\Lambda^{k}_h)^{-1}\left(\sum_{\tau=1}^{k-1}\epsilon^{k,\tau,j}_{h} \phi(s_h^\tau,a_h^\tau) + \xi^{k,j}_{h}\right).
\end{align*}

Since $\epsilon^{k,\tau,j}_{h} \sim N(0,\sigma^2)$, note that for $\tau\in[k-1]$,
\begin{equation*}
    \epsilon^{k,\tau,j}_{h} \phi(s_h^\tau,a_h^\tau) \sim N(0, \sigma^2\phi(s_h^\tau,a_h^\tau)\phi(s_h^\tau,a_h^\tau)^\top).
\end{equation*}

Now, since $\xi^{k,j}_{h}\sim \cN(0,\sigma^2\lambda I_d)$,
\begin{align*}
    (\Lambda^{k}_h)^{-1}\left(\sum_{\tau=1}^{k-1}\epsilon^{k,\tau,j}_{h} \phi(s_h^\tau,a_h^\tau) + \xi^{k,j}_{h}\right) &\sim (\Lambda^{k}_h)^{-1} \cdot N\left(0, \sigma^2\left(\sum_{\tau=1}^{k-1}\phi(s_h^\tau,a_h^\tau)\phi(s_h^\tau,a_h^\tau)^\top +\lambda I_d \right)\right)\\
    &\sim (\Lambda^{k}_h)^{-1} \cdot N\left(0,\sigma^2 \Lambda_h^k\right)\\
    &\sim N(0,\sigma^2(\Lambda_h^k)^{-1}).
\end{align*}
Thus, we have 
\begin{equation*}
    \zeta_h^{k,j} = \Tilde{\theta}^{k,j}_h - \Hat{\theta}^{k}_h\sim N(0,\sigma^2(\Lambda_h^k)^{-1}).
\end{equation*}
\end{proof}

\begin{lemma}[Lemma B.1 in \cite{jin2019provably}]\label{Lemma:generic-weight-vector-bound}
Under Definition ~\ref{definition-linear-MDP} of linear MDP, for any fixed policy $\pi$, let $\{\theta_h^\pi\}_{h \in [H]}$ be the corresponding weights such that $Q_h^\pi(s,a) = \la \phi(s,a),\theta_h^\pi \ra$ for all $(s,a,h)\in \cS\times\cA\times[H]$. Then for all $h\in[H]$, we have
\begin{equation*}
    \|\theta_h^\pi\| \leq 2H\sqrt{d}.
\end{equation*}
\end{lemma}

 Our next lemma states that the unperturbed estimated weight $\hat{\theta}_h^k$ is bounded. 
\begin{lemma}\label{lemma:weight-vector-bound}
For any $(k,h) \in [K]\times[H]$, the unperturbed estimated weight $\hat{\theta}^k_h$ in Definition~\ref{Definition:unperturbed-regressor} satisfies
\begin{equation*}
    \|\hat{\theta}^k_h\| \leq  2H\sqrt{kd/\lambda}.
\end{equation*}
\end{lemma}
\begin{proof}
We have
\begin{equation*}
    \begin{split}
       \bigl \|\hat{\theta}^k_h \bigr\| &=  \Bigl \|(\Lambda^{k}_h)^{-1} \sum_{\tau=1}^{k-1} [r^\tau_{h}(s_h^\tau,a_h^\tau) +   V_{h+1}^k(s_{h+1}^\tau)]\cdot \phi(s_h^\tau,a_h^\tau) \Bigr \| \\
       &=  \Bigl \|(\Lambda^{k}_h)^{-1} \sum_{\tau=1}^{k-1} [r^\tau_{h}(s_h^\tau,a_h^\tau) + \max_{a \in \cA} Q_{h+1}^{k}(s_{h+1}^\tau,a)]\cdot \phi(s_h^\tau,a_h^\tau) \Bigr \| \\
       & \leq \frac{1}{\sqrt{\lambda}}\sqrt{k-1} \Bigl(\sum_{\tau=1}^{k-1} \bigl\|[r^\tau_{h}(s_h^\tau,a_h^\tau) + \max_{a \in \cA} Q_{h+1}^{k}(s_{h+1}^\tau,a)]\cdot \phi(s_h^\tau,a_h^\tau)\bigr\|^2_{(\Lambda^{k}_h)^{-1}} \Bigr)^{1/2}\\
       & \leq \frac{2H}{\sqrt{\lambda}}\sqrt{k-1}\Bigl(\sum_{\tau=1}^{k-1} \|\phi(s_h^\tau,a_h^\tau)\|^2_{(\Lambda^{k}_h)^{-1}} \Bigr)^{1/2}\\
       &\leq 2H\sqrt{kd/\lambda}.
    \end{split}
\end{equation*}

Here, the first inequality follows from Lemma \ref{lemma:linear-algebra-norm-ineq}. The second inequality follows from the truncation of $Q_h^k$ to the range $[0,H-h+1]$ in Line 11  of Algorithm \ref{Algorithm-linear MDP reward perturbation}. The last inequality is due to Lemma \ref{Lemma:D.1-chijin}.
\end{proof}

For the ease of exposition, we now define the values $\beta_k(\delta)$, $\nu_k(\delta)$ and $\gamma_k(\delta)$ which we use to define our high confidence bounds.
\begin{definition}[Noise bounds]\label{def:noise-const}

For any $\delta > 0$ and some large enough constants $c_1$,$c_2$ and $c_3$, let
\begin{align*} 
\sqrt{\beta_k(\delta)} &\mydef  c_1 H\sqrt{d\log(Hdk/\delta)}, \\
\sqrt{\nu_k(\delta)} &\mydef  c_2 H\sqrt{d\log(Hdk/\delta)}, \\ 
\sqrt{\gamma_k(\delta)} &\mydef  c_3 \sqrt{d\nu_k(\delta)\log(d/\delta)}.
\end{align*}
\end{definition}

\begin{definition}[Noise distribution]\label{def:noise-distribution}

In Algorithm \ref{Algorithm-linear MDP reward perturbation}, we set the following values for $\sigma$
\begin{equation*}
    \sigma_k = 2\sqrt{\nu_k(\delta)}.\\
\end{equation*}
Thus for all $j \in [M]$, we have,
$$\{\xi^{k,j}_{h}\} \sim \mathcal{N} \bigl (0,4\nu_k(\delta)(\Lambda^{k}_h)^{-1} \bigr).$$
\end{definition}

Now, we define some events based on the characterization of the random variable $\zeta_h^{k,j}$ as defined in Definition~\ref{Definition:unperturbed-regressor}.
\begin{definition}[Good events]\label{def:good-event}
For any $\delta > 0$, we define the following random events
\begin{align*}
    \mathcal{G}_{h}^k(\zeta, \delta) &\mydef \Bigl\{ \max_{j \in [M]}\|\zeta^{k,j}_{h}\|_{\Lambda_{h}^k} \leq \sqrt{\gamma_k(\delta)}\Bigr\},\\
    \mathcal{G}(K, H, \delta) &\mydef \bigcap_{k \leq K}\bigcap_{h \leq H}\mathcal{G}^k_h(\zeta,\delta).
\end{align*}
\end{definition}

 Next, we present our main concentration lemma in this section. 
\begin{lemma}\label{Lemma:norm_conditioned_high_prob_bound_bar_eta}
Let $\lambda = 1$ in  Algorithm \ref{Algorithm-linear MDP reward perturbation}. For any fixed $\delta > 0$,  conditioned on the event $\mathcal{G}(K, H, \delta)$, we have for all $(k,h) \in [K]\times[H]$,

\begin{equation}
    \Bigl\|\sum_{\tau=1}^{k-1}\phi(s_h^\tau,a_h^\tau) \bigl [\bigl(V^{k}_{h+1} - \mathbb{P}_h V^{k}_{h+1}\bigr) (s_h^\tau, a_h^\tau) \bigr] \Bigr \|_{(\Lambda_h^k)^{-1}} \leq c_1 H\sqrt{d\log{(Hdk/\delta)}},
\end{equation}
 with probability at least $1-\delta$ for some constant $c_1 > 0 $.
\end{lemma}
\begin{proof}

From Lemma \ref{lemma:weight-vector-bound}, we know, for all $(k,h) \in [K]\times[H]$, we have  $\|\hat{\theta}^k_h\| \leq  2H\sqrt{kd/\lambda}$.
In addition, by construction of $\Lambda_{h+1}^k$, the minimum eigenvalue of $\Lambda_{h+1}^k$ is lower bounded by $\lambda$. Thus we have $\sqrt{\lambda}\|\zeta^{k,j}_{h+1}\| \leq \|\zeta^{k,j}_{h+1}\|_{\Lambda_{h+1}^k} \leq \sqrt{\gamma_k(\delta)}$. Finally, triangle inequality implies, $\|\Tilde{\theta}_{h+1}^{k,j}\| =\|\hat{\theta}_{h+1}^k + \zeta^{k,j}_{h+1}\| \leq 2H\sqrt{kd/\lambda} + \sqrt{\gamma_k(\delta)/\lambda}$ for all $j\in [M]$. Combining Lemma \ref{lemma:D.4-chi-jin} and Lemma \ref{lemma:covering_number_of_V(.)}, we have that, for any $\varepsilon > 0$ and $\delta >0$, with probability at least $1-\delta$,
\begin{equation}
    \begin{aligned}[b]
    & \Bigl\|\sum_{\tau=1}^{k-1}\phi(s_h^\tau,a_h^\tau)[\bigl(V^{k}_{h+1} - \mathbb{P}_hV^{k}_{h+1}\bigr) (s_h^\tau, a_h^\tau)]\Bigr\|_{(\Lambda_h^k)^{-1}} \\ 
    &\leq \Bigl(4H^2\Bigl[\frac{d}{2}\log\Bigl(\frac{k +\lambda}{\lambda}\Bigr) + d\log\Bigl(1+ \frac{4H\sqrt{kd/\lambda} + 2\sqrt{\gamma_k(\delta)/\lambda}}{\varepsilon}\Bigr)  + \log\frac{1}{\delta}\Bigr]   + \frac{8k^2\varepsilon^2}{\lambda}\Bigr)^{1/2} \\
        &\leq \Bigl(4H^2\Bigl[\frac{d}{2}\log\Bigl(\frac{k+\lambda}{\lambda}\Bigr) + d\log\Bigl( \frac{3(2H\sqrt{kd/\lambda} + \sqrt{\gamma_k(\delta)/\lambda})}{\varepsilon}\Bigr)   + \log\frac{1}{\delta}\Bigr] + \frac{8k^2\varepsilon^2}{\lambda}\Bigr)^{1/2} \\
        &\leq 2H\Bigl[\frac{d}{2}\log\Bigl(\frac{k+\lambda}{\lambda}\Bigr) + d\log\Bigl( \frac{3(2H\sqrt{kd/\lambda} + \sqrt{\gamma_k(\delta)/\lambda})}{\varepsilon}\Bigr) + \log\frac{1}{\delta}\Bigr]^{1/2}  + \frac{2\sqrt{2}k\varepsilon}{\sqrt{\lambda}} \\
        &\leq 2H\sqrt{d}\Bigl[\frac{1}{2}\log\Bigl(\frac{k+\lambda}{\lambda}\Bigr) + \log\Bigl( \frac{3(2H\sqrt{kd/\lambda} + \sqrt{\gamma_k(\delta)/\lambda})}{\varepsilon}\Bigr)  + \log\frac{1}{\delta}\Bigr]^{1/2}  + \frac{2\sqrt{2}k\varepsilon}{\sqrt{\lambda}}.
    \end{aligned}
\end{equation}

Setting $\lambda = 1, \, \varepsilon = H\sqrt{d}/k$ and substituting $\sqrt{\gamma_k(\delta)} =   c_3 \sqrt{d\nu_k(\delta)\log(d/\delta)} \leq c_4 Hd \log{(Hdk/\delta)} $ for some constant $c_4>0$, we get 
\begin{equation}\label{Eq:high_prob_bound_bar_eta}
    \begin{aligned}[b]
    & \Bigl\|\sum_{\tau=1}^{k-1}\phi(s_h^\tau,a_h^\tau)[\bigl(V^{k}_{h+1} - \mathbb{P}_hV^{k}_{h+1}\bigr) (s_h^\tau, a_h^\tau)]\Bigr\|_{(\Lambda_h^k)^{-1}} \\
    &\leq 2H\sqrt{d}\Biggl[\frac{1}{2}\log(k+1) + \log(1/\delta)   + \log \frac{3k[2H\sqrt{dk} + c_4 Hd \log{(Hdk/\delta)}]}{H\sqrt{d}} \Biggr]^{1/2}  + 2\sqrt{2}H\sqrt{d}\\
    & \leq c_1 H\sqrt{{d}\log{(Hdk/\delta)}},
    \end{aligned}
\end{equation}

for some constant $c_1 >0$.

\end{proof}

\begin{lemma}\label{Lemma:good_event_hat_theta_minus_r_plus_Bar_v}
 
Let $\lambda = 1$ in Algorithm \ref{Algorithm-linear MDP reward perturbation}. For any $\delta > 0 $, conditioned on the event $\mathcal{G}(K, H, \delta)$, for any $(h, k) \in  [H] \times [K]$ and $(s,a) \in \cS\times\cA$, we have

\begin{equation*}
    \bigl| \phi(s,a)^\top \hat{\theta}^{k}_h - r_h^k(s,a) - \mathbb{P}_h V^{k}_{h+1} (s,a) \bigr| \leq c_2 H\sqrt{d\log(Hdk/\delta)} \bigl\| \phi(s,a) \bigr \|_{(\Lambda^{k}_h)^{-1}},
\end{equation*}

with probability $1-\delta$, where $c_2> 0$ is a constant.
\end{lemma}

\begin{proof}
Let us denote the inner product over $\cS$ by $\la \cdot, \cdot \ra_\cS$. Using linear MDP assumption for transition kernel from Definition~\ref{definition-linear-MDP}, we get
\begin{align}\label{Eq:PV(s,a)}
        \mathbb{P}_h V^{k}_{h+1}(s,a) &= \phi(s,a)^\top \la \mu_h, V^{k}_{h+1} \ra_\cS \notag\\
        &= \phi(s,a)^\top (\Lambda^{k}_h)^{-1}\Lambda^{k}_h\la \mu_h, V^{k}_{h+1}\ra_\cS \notag\\
        &= \phi(s,a)^\top (\Lambda^{k}_h)^{-1}\Bigl( \sum_{\tau=1}^{k-1}\phi(s_h^\tau,a_h^\tau)\phi(s_h^\tau,a_h^\tau)^\top + \lambda {I}\Bigr)\la \mu_h, V^{k}_{h+1} \ra_\cS \notag\\
        &= \phi(s,a)^\top (\Lambda^{k}_h)^{-1}\Bigl( \sum_{\tau=1}^{k-1}\phi(s_h^\tau,a_h^\tau)\phi(s_h^\tau,a_h^\tau)^\top \la \mu_h, V^{k}_{h+1} \ra_\cS + \lambda {I} \la \mu_h, V^{k}_{h+1} \ra_\cS\Bigr) \notag\\
        &= \phi(s,a)^\top (\Lambda^{k}_h)^{-1}\Bigl( \sum_{\tau=1}^{k-1}\phi(s_h^\tau,a_h^\tau)(\mathbb{P}_h V^{k}_{h+1}) (s_h^\tau, a_h^\tau) + \lambda {I} \la \mu_h, V^{k}_{h+1} \ra_\cS\Bigr),
\end{align}
where in the last line we rely on the definition of $\mathbb{P}_h$.

Using \eqref{Eq:PV(s,a)} we obtain,
\begin{align}\label{Eq:decomp}
        \phi(s,a)^\top\hat{\theta}^k_h - r_h^k(s,a) - (\mathbb{P}_h V^{k}_{h+1})(s,a) &= \phi(s,a)^\top (\Lambda^{k}_h)^{-1}\sum_{\tau=1}^{k-1} \bigl[r^\tau_{h}(s_h^\tau,a_h^\tau) + V_{h+1}^k(s_{h+1}^\tau)\bigr]\cdot \phi(s_h^\tau,a_h^\tau) - r_h^k(s,a) \notag\\
        & \qquad  - \phi(s,a)^\top (\Lambda^{k}_h)^{-1}\Bigl( \sum_{\tau=1}^{k-1}\phi(s_h^\tau,a_h^\tau)(\mathbb{P}_hV^{k}_{h+1}) (s_h^\tau, a_h^\tau) + \lambda {I} \la \mu_h, V^{k}_{h+1} \ra_\cS\Bigr)\notag\\
        &= \underbrace{\phi(s,a)^\top (\Lambda^{k}_h)^{-1}\Bigl( \sum_{\tau=1}^{k-1}\phi(s_h^\tau,a_h^\tau)\bigl[\bigl(V^{k}_{h+1} - \mathbb{P}_h V^{k}_{h+1}\bigr) (s_h^\tau, a_h^\tau)\bigr] \Bigr)}_{\text{(i)}} \notag\\
        & \qquad + \underbrace{\phi(s,a)^\top (\Lambda^{k}_h)^{-1}\Bigl( \sum_{\tau=1}^{k-1} r_h^\tau (s_h^\tau,a_h^\tau) \phi(s_h^\tau,a_h^\tau)\Bigr) - r_h^k(s,a)}_{\text{(ii)}} \notag\\
        & \qquad - \underbrace{\lambda\phi(s,a)^\top (\Lambda^{k}_h)^{-1}\la \mu_h, V^{k}_{h+1}\ra_\cS}_{\text{(iii)}}.  
\end{align}

\vspace{4pt}

In the following we will analyze the each of the three terms in \eqref{Eq:decomp} separately and derive high probability bound for each of them.

\noindent
\textbf{Term (i).}
Since $(\Lambda_h^k)^{-1} \succ 0$, by Cauchy-Schwarz inequality and Lemma \ref{Lemma:norm_conditioned_high_prob_bound_bar_eta}, with probability at least $1-\delta$, we have
\begin{equation}\label{Eq.term_i}
    \begin{aligned}[b]
    &\phi(s,a)^\top (\Lambda^{k}_h)^{-1}\Bigl( \sum_{\tau=1}^{k-1}\phi(s_h^\tau,a_h^\tau)\bigl[\bigl(V^{k}_{h+1} - \mathbb{P}_hV^{k}_{h+1}\bigr) (s_h^\tau, a_h^\tau)\bigr] \Bigr) \\
        & \leq  \Bigl \|\sum_{\tau=1}^{k-1}\phi(s_h^\tau,a_h^\tau)\bigl[\bigl(V^{k}_{h+1} - \mathbb{P}_hV^{k}_{h+1}\bigr) (s_h^\tau, a_h^\tau)\bigr] \Bigr \|_{(\Lambda_h^k)^{-1}} \bigl\|\phi(s,a) \bigr \|_{(\Lambda_h^k)^{-1}} \\
        & \leq \sqrt{\beta_k(\delta)} \bigl \|\phi(s,a) \bigr \|_{(\Lambda_h^k)^{-1}}.
    \end{aligned}
\end{equation}

\vspace{4pt}
\noindent
\textbf{Term (ii).} Note that
\begin{equation}\label{Eq.term_ii}
    \begin{aligned}[b]
        & \phi(s,a)^\top (\Lambda^{k}_h)^{-1}\Bigl( \sum_{\tau=1}^{k-1} r_h^\tau (s_h^\tau,a_h^\tau) \phi(s_h^\tau,a_h^\tau)\Bigr) - r_h^k(s,a) \\ 
        &= \phi(s,a)^\top (\Lambda^{k}_h)^{-1}\Bigl( \sum_{\tau=1}^{k-1} r_h^\tau (s_h^\tau,a_h^\tau) \phi(s_h^\tau,a_h^\tau)\Bigr) - \phi(s,a)^\top w_h  \\
        &= \phi(s,a)^\top (\Lambda^{k}_h)^{-1}\Bigl( \sum_{\tau=1}^{k-1} r_h^\tau (s_h^\tau,a_h^\tau) \phi(s_h^\tau,a_h^\tau) - \Lambda^{k}_h w_h\Bigr)  \\
        &= \phi(s,a)^\top (\Lambda^{k}_h)^{-1}\Bigl( \sum_{\tau=1}^{k-1} r_h^\tau (s_h^\tau,a_h^\tau) \phi(s_h^\tau,a_h^\tau)   - \sum_{\tau=1}^{k-1}  \phi(s_h^\tau,a_h^\tau)\phi(s_h^\tau,a_h^\tau)^\top w_h 
        - \lambda {I}w_h\Bigr)  \\
        &= \phi(s,a)^\top (\Lambda^{k}_h)^{-1}\Bigl( \sum_{\tau=1}^{k-1} r_h^\tau (s_h^\tau,a_h^\tau) \phi(s_h^\tau,a_h^\tau) - \sum_{\tau=1}^{k-1}  \phi(s_h^\tau,a_h^\tau)r_h^\tau (s_h^\tau,a_h^\tau) - \lambda {I}w_h\Bigr)  \\
        &= -\lambda\phi(s,a)^\top (\Lambda^{k}_h)^{-1} w_h,
    \end{aligned}
\end{equation}
where in the penultimate step, we used the fact $r_h(s,a) = \langle \phi(s,a), w_h \rangle$ from Definition~\ref{definition-linear-MDP}. Applying Cauchy-Schwarz inequality we obtain, 
\begin{align} \label{Eq:cauchy-term-2}
        -\lambda\phi(s,a)^\top (\Lambda^{k}_h)^{-1} w_h &\leq \lambda \|\phi(s,a)\|_{(\Lambda^{k}_h)^{-1}} \|w_h\|_{(\Lambda^{k}_h)^{-1}} \notag\\
        &\leq \sqrt{\lambda}\|\phi(s,a)\|_{(\Lambda^{k}_h)^{-1}} \|w_h\|_{2} \notag\\
        &\leq \sqrt{\lambda d}\|\phi(s,a)\|_{(\Lambda^{k}_h)^{-1}}.
\end{align}

Here the second inequality follows by observing that the smallest eigenvalue of $\Lambda^{k}_h$ is at least $\lambda$ and thus the largest eigenvalue of $(\Lambda^{k}_h)^{-1}$ is at most $1/\lambda$. The last inequality follows from Definition~\ref{definition-linear-MDP}. Combining \eqref{Eq.term_ii} and \eqref{Eq:cauchy-term-2} we get
\begin{equation}\label{Eq.term_ii_bound}
    \phi(s,a)^\top (\Lambda^{k}_h)^{-1}\Bigl( \sum_{\tau=1}^{k-1} r_h^\tau (s_h^\tau,a_h^\tau) \phi(s_h^\tau,a_h^\tau)\Bigr) - r_h^k(s,a) \leq \sqrt{\lambda d}\|\phi(s,a)\|_{(\Lambda^{k}_h)^{-1}}.
\end{equation}

\vspace{4pt}
\noindent
\textbf{Term (iii).} 
Similar to (\ref{Eq:cauchy-term-2}), applying Cauchy-Schwarz inequality, we get

\begin{align}\label{Eq.term_iii}
    -\lambda\phi(s,a)^\top (\Lambda^{k}_h)^{-1}\la \mu_h, V^{k}_{h+1} \ra_\cS &\leq  \lambda \|\phi(s,a)\|_{(\Lambda^{k}_h)^{-1}} \|\la \mu_h, V^{k}_{h+1} \ra_\cS\|_{(\Lambda^{k}_h)^{-1}} \notag\\
        & \leq \sqrt{\lambda}\|\phi(s,a)\|_{(\Lambda^{k}_h)^{-1}}  \|\la \mu_h, V^{k}_{h+1} \ra_\cS\|_2 \notag\\
        & \leq \sqrt{\lambda}\|\phi(s,a)\|_{(\Lambda^{k}_h)^{-1}}\Bigl(\sum_{\tau=1}^d \|\mu_h^\tau\|_1^2\Bigr)^\frac{1}{2} \|V^{k}_{h+1}\|_{\infty} \notag\\
        & \leq H\sqrt{\lambda d}\|\phi(s,a)\|_{(\Lambda^{k}_h)^{-1}}.
\end{align}

Here the second inequality follows using the same observation we did for \textbf{term (ii)}. The last inequality follows from $\sum_{\tau=1}^d \|\mu_h^\tau\|_1^2 \leq d$ in Definition~\ref{definition-linear-MDP} and the clipping operation performed in Line \ref{Alg:min-max-for-Q} of Algorithm \ref{Algorithm-linear MDP reward perturbation}. Now combining (\ref{Eq.term_i}), (\ref{Eq.term_ii_bound}) and (\ref{Eq.term_iii}), and letting $\lambda = 1$, we get, 

\begin{align}
    \bigl | \phi(s,a)^\top \hat{\theta}^{k}_h - r_h^k(s,a) - \mathbb{P}_h V^{k}_{h+1} (s,a) \bigr| &\leq (\sqrt{\beta_k(\delta)} + H\sqrt{d} + \sqrt{d})\| \phi(s,a) \|_{(\Lambda^{k}_h)^{-1}}\\
    &= (c_1 H\sqrt{d\log(Hdk/\delta)} + H\sqrt{d} + \sqrt{d})\| \phi(s,a) \|_{(\Lambda^{k}_h)^{-1}}\\
    &\leq c_2 H\sqrt{d\log(Hdk/\delta)}\| \phi(s,a) \|_{(\Lambda^{k}_h)^{-1}},
\end{align}
with probability $1-\delta$ for some constant $c_2 > 0$.

In addition, If we set $\theta_h^k: \phi(\cdot,\cdot)^\top\theta_h^k=r_h^k(\cdot,\cdot)+\mathbb{P}_{h}V^{k}_{h+1}(\cdot,\cdot)$ to be the true parameter and $\Delta\theta_h^k=\theta_h^k-\Hat{\theta}_h^k$ to be the regression error, then from the analysis above we can derive that $\|\Delta\theta_h^k\|_{\Lambda_h^k}\leq\sqrt{\nu_k(\delta)}=c_2 H\sqrt{d\log(Hdk/\delta)}$.
\end{proof}

\begin{lemma}[stochastic upper confidence bound]\label{Lemma:model_prediction_error_leq_0}

Let $\lambda = 1$ in Algorithm \ref{Algorithm-linear MDP reward perturbation}. For any $\delta > 0 $, conditioned on the event $\mathcal{G}(K, H, \delta)$, for any $(h, k) \in  [H] \times [K]$ and $(s,a) \in \cS\times\cA$, with probability at least $1- (\delta + c_0^M)$, we have
\begin{equation*}
l^{k}_{h}(s,a) \leq 0,  
\end{equation*}
and
\begin{equation*}
    -l^{k}_{h}(s,a) \leq \Bigl(\sqrt{\nu_k(\delta)} + \sqrt{\gamma_k(\delta)}\Bigr) \bigl\|\phi(s,a) \bigr \|_{(\Lambda^{k}_h)^{-1}},
\end{equation*}
where $c_0 = \Phi(1)$.
\end{lemma}

\begin{proof}
Applying Lemma \ref{Lemma:good_event_hat_theta_minus_r_plus_Bar_v}, for any $(h, k) \in  [H] \times [K]$ and $(s,a) \in \cS\times\cA$, we have, 
\begin{align}\label{Eq:application_lemma_l_leq0_in}
     \bigl |r^k_h (s,a) +\mathbb{P}_{h}V^{k}_{h+1}(s,a) - \phi(s,a)^\top \hat{\theta}_h^k \bigr|  &\leq c_2 H\sqrt{d\log(Hdk/\delta)}\\
     &=\sqrt{\nu_k(\delta)} \bigl \|\phi(s,a) \bigr \|_{(\Lambda^{k}_h)^{-1}},
\end{align}
with probability at least $1-\delta$.

As we are conditioning on the event $\mathcal{G}(K, H, \delta)$, for any $(h, k) \in  [H] \times [K]$ and $(s,a) \in \cS\times\cA$, we have

\begin{equation}\label{Eq:GKHdelta}
      \max_{j \in [M]} \bigl |\phi(s,a)^\top\zeta_h^{k,j} \bigr|  \leq \sqrt{\gamma_k(\delta)} \bigl \|\phi(s,a) \bigr \|_{(\Lambda^{k}_h)^{-1}}.
\end{equation}

Now from the definition of model prediction error, using \eqref{Eq:application_lemma_l_leq0_in} and \eqref{Eq:GKHdelta},  we get, with probability $1-\delta$,
\begin{equation}\label{Eq:unionbound_prob1_minus2}
    \begin{aligned}[b]
         - l^{k}_{h}(s,a)  &= Q^{k}_h(s,a) - r^k_h (s,a) - \mathbb{P}_{h}V^{k}_{h+1}(s,a)   \\
         & = \min \{\max_{j \in [M]}\phi(s,a)^\top (\hat{\theta}_h^k + \zeta_h^{k,j}), H  \} - r^k_h (s,a) - \mathbb{P}_{h}V^{k}_{h+1}(s,a) \\
         & \leq \max_{j \in [M]}\phi(s,a)^\top (\hat{\theta}_h^k + \zeta_h^{k,j}) - r^k_h (s,a) - \mathbb{P}_{h}V^{k}_{h+1}(s,a) \\
         &= \max_{j \in [M]}\phi(s,a)^\top  \zeta_h^{k,j} - \bigl(r^k_h (s,a) +\mathbb{P}_{h}V^{k}_{h+1}(s,a) - \phi(s,a)^\top \hat{\theta}_h^k \bigr) \\ 
         &\leq \bigl |r^k_h (s,a) +\mathbb{P}_{h}V^{k}_{h+1}(s,a) - \phi(s,a)^\top \hat{\theta}_h^k \bigr | + \max_{j \in [M]} \bigl |\phi(s,a)^\top\zeta_h^{k,j} \bigr | \\
         & \leq \Bigl(\sqrt{\nu_k(\delta)} + \sqrt{\gamma_k(\delta)}\Bigr) \bigl \|\phi(s,a) \bigr \|_{(\Lambda^{k}_h)^{-1}},
    \end{aligned}
\end{equation}

Set $\theta_h^k: \phi(\cdot,\cdot)^\top\theta_h^k=r_h^k(\cdot,\cdot)+\mathbb{P}_{h}V^{k}_{h+1}(\cdot,\cdot)$ to be the true parameter and $\Delta\theta_h^k=\theta_h^k-\Hat{\theta}_h^k$ to be the regression error. By the concentration part, conditioning on good events,
we have $\|\Delta\theta_h^k\|_{\Lambda_h^k}\leq\sqrt{\nu_k(\delta)}$ and $\|\xi^{k,j}_{h}\|_{\Lambda_h^k}\leq\sqrt{\gamma_k(\delta)}$ for all $j\in[M]$.

For all $(h,k)\in[H]\times[K]$ and any $(s,a)\in\cS\times\cA$, we have
\begin{align*}
l^{k}_{h}(s,a) &=    r^k_h (s,a) +\mathbb{P}_{h}V^{k}_{h+1}(s,a) - Q^{k}_h(s,a)\\
    &= r^k_h (s,a) +\mathbb{P}_{h}V^{k}_{h+1}(s,a) - \min \bigl \{H  , \max_{j\in [M]}\phi(s, a)^\top(\hat{\theta}^k_h + \xi^{k,j}_{h}) \bigr \}^+ \\
    &\leq \max\{\phi(s,a)^\top \Delta\theta_h^k-\max_{j\in[M]}\phi(s, a)^\top \xi^{k,j}_{h},0\}
\end{align*}

Now we prove that with high probability, $\max_{j\in[M]}\phi(s, a)^\top \xi^{k,j}_{h}-\phi(s,a)^\top \Delta\theta_h^k\geq0$ for all $(s,a)\in\cS\times\cA$. Note that the inequality still holds if we scale $\phi(s,a)$. Now we assume all $\phi(s,a)$ satisfy $\|\phi(s,a)\|_{(\Lambda_h^k)^{-1}}=1$. Define $\cC(\epsilon)$ to be a $\epsilon$-cover of the ellipsoid $\{\phi|\|\phi\|_{(\Lambda_h^k)^{-1}}=1\}$ with respect to norm $\|\cdot\|_{(\Lambda_h^k)^{-1}}$ and $\log|\cC(\epsilon)|=\Tilde{O}(d\log(\frac{1}{\epsilon}))$. For all $j \in [M]$, we have,
$$\{\xi^{k,j}_{h}\} \sim \cN \bigl (0,4\nu_k(\delta)(\Lambda^{k}_h)^{-1} \bigr).$$

Thus, for all $j \in [M]$ and for all $\phi\in\cC(\epsilon)$ , we have 
$$\bigl\{\phi^\top\xi^{k,j}_{h}\bigr\} \sim \cN \Bigl (0,4\nu_k(\delta)\|\phi\|^2_{(\Lambda^{k}_h)^{-1}}\Bigr).$$
Now, for all $j\in [M]$ and for all $\phi\in\cC(\epsilon)$, we have
\begin{equation*}
    \mathbb{P}\Bigl(\phi^\top \xi^{k,j}_{h} - 2\sqrt{\nu_k(\delta)}\|\phi\|_{(\Lambda^{k}_h)^{-1}} \geq 0 \Bigr) = \Phi(-1).
\end{equation*}
Now
\begin{equation}\label{eq:ineq_with_M}
    \begin{aligned}[b]
        \mathbb{P}\Bigl(\max_{j\in [M]}\phi^\top \xi^{k,j}_{h} - 2\sqrt{\nu_k(\delta)}\|\phi\|_{(\Lambda^{k}_h)^{-1}} \geq 0 \Bigr) &\geq 1 - (1-\Phi(-1))^M\\
        &= 1 - \Phi(1)^M \\
        &= 1 - c_0^M,
    \end{aligned}
\end{equation}
By union bound, with probability $1-|\cC(\epsilon)|c_0^M$, the above bound holds for all elements in $\cC$ simultaneously.

Now condition on the previous event, for $\phi=\phi(s,a)$, we can find a $\phi'\in\cC(\epsilon)$ such that $\|\phi-\phi'\|_{(\Lambda^{k}_h)^{-1}}\leq\epsilon$. Define $\Delta\phi=\phi-\phi'$.

\begin{align*}
    \phi^\top \xi^{k,j}_{h}-\phi^\top \Delta\theta_h^k&=\phi'^\top \xi^{k,j}_{h}-\phi'^\top \Delta\theta_h^k+\Delta\phi^\top \xi^{k,j}_{h}+\Delta\phi^\top\Delta\theta_h^k\\
    &\geq \phi'^\top \xi^{k,j}_{h} - 2\sqrt{\nu_k(\delta)}\|\phi'\|_{(\Lambda^{k}_h)^{-1}} + \sqrt{\nu_k(\delta)}\|\phi'\|_{(\Lambda^{k}_h)^{-1}}-\epsilon\|\xi^{k,j}_{h}\|_{\Lambda^{k}_h}-\epsilon\|\Delta\theta_h^k\|_{\Lambda^{k}_h}\\
    &\geq \phi'^\top \xi^{k,j}_{h} - 2\sqrt{\nu_k(\delta)}\|\phi'\|_{(\Lambda^{k}_h)^{-1}} + \sqrt{\nu_k(\delta)}\|\phi'\|_{(\Lambda^{k}_h)^{-1}}-\epsilon\sqrt{\gamma_k(\delta)}-\epsilon\sqrt{\nu_k(\delta)}\\
\end{align*}

Set $\epsilon=\frac{\sqrt{\nu_k(\delta)}}{\sqrt{\gamma_k(\delta)}+\sqrt{\nu_k(\delta)}}=\Tilde{O}(\frac{1}{\sqrt{d}})$ and we have, with probability $1-|\cC(\epsilon)|c_0^M$,
\begin{align*}
    \max_{j\in [M]}\phi^\top \xi^{k,j}_{h}-\phi^\top \Delta\theta_h^k&\geq \max_{j\in [M]}\phi'^\top \xi^{k,j}_{h} - 2\sqrt{\nu_k(\delta)}\|\phi'\|_{(\Lambda^{k}_h)^{-1}}\\
    &\geq 0.
\end{align*}

Finally we have conditioning on good event $\mathcal{G}(K, H, \delta)$, with probability at least $1-|\cC(\epsilon)|c_0^M$, for all $(s,a)\in\mathcal{S}\times\cA$ , $l_h^k(s,a)\leq0$. As $\log|\cC(\epsilon)|=\Tilde{O}(d\log(\frac{1}{\epsilon}))$, we can set $M=\Tilde{O}(\frac{d\log(1/\epsilon\delta)}{\log(1/c_0)})=\Tilde{O}(d)$ to have probability $1-\delta$.

\end{proof}

\subsection{Regret Bound}

\begin{definition}[Filtrations]\label{def:filtrations} 

We denote the $\sigma$-algbera generated by the set $\cG$ using $\sigma(\cG)$.  We define the following filtrations: 
\begin{align*} 
    \cF^k &\mydef \sigma \left(\{(s_t^i, a_t^i, r_t^i)\}_{\{i,t\} \in [k-1]\times[H]}\,\ \bigcup \,\ \{\xi^{i,j}_{t}\}_{\{i,t,j\} \in [k-1]\times[H]\times[M]}\right),\\
    \cF_{h,1}^k &\mydef \sigma \left(\cF^k \,\ \bigcup \,\ \{(s_t^k, a_t^k, r_t^k)\}_{t\in[h]}\,\ \bigcup \,\ \{\xi^{k,j}_{t}: t \leq h,\,\ 1\leq j \leq M\}\right),\\
    \cF_{h,2}^k &\mydef \sigma\left(\cF_{h,1}^k \,\ \bigcup \,\ \{x_{h+1}^k\}\right).
\end{align*}
\end{definition}

\begin{lemma}[Lemma 4.2 in \cite{cai2019provably}]\label{lemma:regret-decom}

It holds that
\begin{equation} \label{Eq:regret_decomp}
\begin{aligned}[b]
\mathrm{Regret}(T) &= \sum_{k=1}^K \left(V_1^{*}(s_1^k) - V_1^{\pi^k}(s_1^k)\right) \\
&= \underbrace{\sum_{k=1}^K \sum_{ t=1}^H \mathbb{E}_{\pi^*} \left[ \la Q_h^k(s_h, \cdot), \pi_h^*(\cdot \mid s_h) - \pi_h^k(\cdot \mid s_h) \ra \biggiven s_1 = s_1^k\right]}_{\text{(i)}} \\
& \qquad + \underbrace{\sum_{k=1}^K \sum_{ t=1}^H \cD_h^k}_{\text{(ii)}}  \ + \ \underbrace{\sum_{k=1}^K \sum_{ t=1}^H \cM_h^k}_{\text{(iii)}} \\
& \qquad + \underbrace{\sum_{k=1}^K \sum_{h=1}^H \left ( \mathbb{E}_{\pi^*}\left[l^{k}_{h}(s_h,a_h) \mid s_1 = s_1^k \right] - l^{k}_{h}(s_h^k,a_h^k) \right)}_{\text{(iv)}},
\end{aligned}
\end{equation}
where
\begin{align}
    \cD_h^k &\coloneqq \la (Q_h^k - Q_h^{\pi^k})(s_h^k, \cdot), \pi_h^k(\cdot,s_h^k)\ra - (Q_h^k-Q_h^{\pi^k})(s_h^k,a_h^k),\\
    \cM_h^k &\coloneqq \mathbb{P}_h( (V_{h+1}^k - V_{h+1}^{\pi^k}))(s_h^k, a_h^k) - (V_{h+1}^k-V_{h+1}^{\pi^k})(s_h^k).
\end{align}

\end{lemma}

\begin{lemma}\label{Lemma:performance_difference_regret}
For the policy $\pi_h^k$ at time-step $k$ of episode $h$, it holds that
\begin{equation}\label{Eq:performance_difference_regret}
    \sum_{k=1}^K \sum_{ t=1}^H \mathbb{E}_{\pi^*} \left[ \la Q_h^k(s_h, \cdot), \pi_h^*(\cdot \mid s_h) - \pi_h^k(\cdot \mid s_h) \ra \mid s_1 = s_1^k\right] \leq 0,
\end{equation}
where $T = HK$.
\end{lemma}

\begin{proof}
Obvious from the observation that $\pi^k_h$ acts greedily with respect to $Q_h^k$. Note that if $\pi_h^k = \pi_h^*$ then the difference is $0$. Else the difference is negative since $\pi_h^k$ is deterministic with respect to its action-values meaning it takes a value of $1$ where $\pi_h^*$ would take a value of $0$ and $Q_h^k$ would have the greatest value at the state-action pair that $\pi_h^k$ equals one.
\end{proof}

\begin{lemma}[Bound on Martingale Difference Sequence]\label{Lemma:Martingale-DS}
For any  $\delta  >0$, it holds with probability $1-2\delta/3$ that
\begin{equation} \label{Eq:martingale-DS}
    \sum_{k=1}^K \sum_{ t=1}^H \cD_h^k + \sum_{k=1}^K \sum_{ t=1}^H \cM_h^k \leq 2\sqrt{2H^2T\log(3/\delta)}.
\end{equation}
\end{lemma}
\begin{proof}
Recall that
\begin{align*}
    \cD_h^k &\coloneqq \la (Q_h^k - Q_h^{\pi^k})(s_h^k, \cdot), \pi_h^k(\cdot,s_h^k)\ra - (Q_h^k-Q_h^{\pi^k})(s_h^k,a_h^k),\\
    \cM_h^k &\coloneqq \mathbb{P}_h( (V_{h+1}^k - V_{h+1}^{\pi^k}))(s_h^k, a_h^k) - (V_{h+1}^k-V_{h+1}^{\pi^k})(s_h^k).
\end{align*}

Note that in line \ref{Alg:min-max-for-Q} of Algorithm \ref{Algorithm-linear MDP reward perturbation}, we truncate $Q_h^k$ to the range $[0,H-h]$. Thus for any $(k,t) \in [K]\times[H]$, we have, $|\cD_h^k| \leq 2H$. Moreover, since $\mathbb{E}[\cD_h^k|\cF_{h,1}^k] = 0$, $\cD_h^k$ is a martingale difference sequence. So, applying Azuma-Hoeffding inequality we have with probability at least $1-\delta/3$, 
\begin{equation}\label{Eq:martingale_bound_D}
    \sum_{k=1}^K \sum_{ t=1}^H \cD_h^k \leq \sqrt{2H^2T\log(3/\delta)}, 
\end{equation}
where $T = KH$.

Similarly, $\cM_h^k$ is a martingale difference sequence since for any $(k,t) \in [K]\times[H]$, $|\cM_h^k| \leq 2H$ and  $\mathbb{E}[\cM_h^k|\cF_{h,1}^k] = 0$. Applying Azuma-Hoeffding inequality we have with probability at least $1-\delta/3$, 

\begin{equation}\label{Eq:martingale_bound_M}
    \sum_{k=1}^K \sum_{ t=1}^H \cM_h^k \leq \sqrt{2H^2T\log(3/\delta)}. 
\end{equation}

Applying union bound on \eqref{Eq:martingale_bound_D} and \eqref{Eq:martingale_bound_M} gives \eqref{Eq:martingale-DS} and completes the proof.

\end{proof}

\begin{lemma} \label{lemma:model-prediction-error-bound}
Let $\lambda = 1$ in Algorithm \ref{Algorithm-linear MDP reward perturbation}. For any $\delta > 0 $, conditioned on the event $\mathcal{G}(K, H, \delta)$, we have,
\begin{equation}\label{Eq:model-prediction-error-bound}
    \sum_{k=1}^K \sum_{h=1}^H \bigl ( \mathbb{E}_{\pi^*}\left[l^{k}_{h}(s_h,a_h) | s_1 = s_1^k \bigr] - l^{k}_{h}(s_h^k,a_h^k) \right) \leq \bigl(\sqrt{\nu_K(\delta)} + \sqrt{\gamma_K(\delta)}\bigr)  \sqrt{2dHT\log(1+K)},
\end{equation}
with probability $1- (\delta + c_0^M)$.
\end{lemma}

\begin{proof}
By Lemma \ref{Lemma:model_prediction_error_leq_0}, with probability $1- (\delta + c_0^M)$  it holds that
\begin{equation}\label{Eq:E_l_expectation}
    \sum_{k=1}^K \sum_{h=1}^H  \mathbb{E}_{\pi^*}\left[l^{k}_{h}(s_h,a_h) | s_1 = s_1^k \right] \leq 0,
\end{equation}

and
\begin{equation} \label{Eq:neg_l_upperbound}
    \begin{aligned}[b]
        \sum_{k=1}^K \sum_{h=1}^H - l^{k}_{h}(s_h^k,a_h^k) &\leq \sum_{k=1}^K \sum_{h=1}^H \Bigl(\sqrt{\nu_k(\delta)} + \sqrt{\gamma_k(\delta)}\Bigr) \bigl\|\phi(s_h^k,a_h^k) \bigr \|_{(\Lambda^{k}_h)^{-1}}\\
        &\leq \Bigl(\sqrt{\nu_K(\delta)} + \sqrt{\gamma_K(\delta)}\Bigr) \sum_{k=1}^K \sum_{h=1}^H \bigl \|\phi(s_h^k,a_h^k)\bigr \|_{(\Lambda^{k}_h)^{-1}}\\
        &\leq \Bigl(\sqrt{\nu_K(\delta)} + \sqrt{\gamma_K(\delta)}\Bigr)  \sum_{h=1}^H \sqrt{K} \Bigl(\sum_{k=1}^K \bigl \|\phi(s_h^k,a_h^k) \bigr \|^2_{(\Lambda^{k}_h)^{-1}}\Bigr)^{1/2} \\
        &\leq \Bigl(\sqrt{\nu_K(\delta)} + \sqrt{\gamma_K(\delta)}\Bigr)  H\sqrt{2dK\log(1+K)}\\
        &=\Bigl(\sqrt{\nu_K(\delta)} + \sqrt{\gamma_K(\delta)}\Bigr)  \sqrt{2dHT\log(1+K)}.
    \end{aligned}
\end{equation}
Here the second inequality follows from the fact that both $\nu_k(\delta)$ and $\gamma_k(\delta)$ are increasinig in $k$. The third and the fourth inequalities follow from Cauchy-Schwarz inequality and Lemma \ref{Lemma:sum_of_feature_yasin}. Combining \eqref{Eq:E_l_expectation} and \eqref{Eq:neg_l_upperbound} completes the proof.
\end{proof}

\begin{lemma}[Good event probability]\label{Lemma:good-event-for-all-step}
For any $K \in \mathbb{N}$ and any  $\delta > 0$, we would have the event $\mathcal{G}(K, H, \delta')$ with probability at least $1-\delta$, where $\delta' = \delta/MT$.
\end{lemma}
\begin{proof}
By Lemma \ref{lemma:max_noise_union_bound}, we have, for any fixed $t$ and $k$, the event $\mathcal{G}^k_h(\xi,\delta')$ occurs with probability at least $1- M\delta'$. Recall from Definition \ref{def:good-event} that, 
\begin{equation*}
    \mathcal{G}(K, H, \delta') = \bigcap_{k \leq K}\bigcap_{h \leq H}\mathcal{G}^k_h(\xi,\delta').
\end{equation*}
Now taking union bound over all $(t,k) \in [H]\times [K]$, we have 
\begin{equation*}
    \mathbb{P}(\bigcap_{k \leq K}\bigcap_{h \leq H}\mathcal{G}^k_h(\xi,\delta')) \geq 1 - MT\delta' = 1- \delta,
\end{equation*}
which completes the proof.
\end{proof}

\begin{theorem} \label{thm:main_theorem} Let  $\lambda = 1$, $\sigma = \Tilde{O}(H\sqrt{d})$ and $M = d\log(\delta/9)/\log c_0$, where $c_0 = \Phi(1)$ and $\delta \in (0,1]$. Under Definition \ref{definition-linear-MDP}, the regret of Algorithm \ref{Algorithm-linear MDP reward perturbation} satisfies

$$\text{Regret}(T) \leq \Tilde{\mathcal{O}}(d^{3/2} H^{3/2} \sqrt{T}),$$
with probability at least $1-\delta$.
\end{theorem}

\begin{proof}[Proof of Theorem \ref{thm:main_theorem}]
Let $\delta' = \delta/9$. From Lemma \ref{Lemma:good-event-for-all-step}, the event $\mathcal{G}(K, H,\delta')$ happens with probability $1- \delta'$. Combining Lemma \ref{lemma:model-prediction-error-bound} and Lemma \ref{Lemma:good-event-for-all-step} we have that the event $\mathcal{G}(K, H,\delta')$ occurs and it holds that

\begin{equation}\label{Eq:ldeltaprime}
    \sum_{k=1}^K \sum_{h=1}^H \left ( \mathbb{E}_{\pi^*}\left[l^{k}_{h}(s_h,a_h) | s_1 = s_1^k \right] - l^{k}_{h}(s_h^k,a_h^k) \right) \leq \Bigl(\sqrt{\nu_K(\delta')} + \sqrt{\gamma_K(\delta')}\Bigr)  \sqrt{2dHT\log(1+K)},
\end{equation}

with probability at least $(1-\delta')(1- (\delta' + c_0^M))$. Note that $c_0^M = \delta'$ and  $(1-\delta')(1- (\delta' + c_0^M)) > 1 - 3\delta' = 1 - \delta/3$.
The martingale inequalities from Lemma \ref{Lemma:Martingale-DS} happens with probability $1-2\delta/3$.

Applying union bound on \eqref{Eq:performance_difference_regret}, \eqref{Eq:martingale-DS} and \eqref{Eq:ldeltaprime} gives the final regret bound of $\tilde{\mathcal{O}}(d^{3/2}H^{3/2}\sqrt{T})$ completes the proof.
\end{proof}

\section{Auxiliary lemmas }
This section presents several auxiliary lemmas and their proofs. 

\subsection{Gaussian Concentration}
\begin{lemma}[Gaussian Concentration \citep{vershynin2018high}] \label{lemma-gaussian-concentration}

Consider a  $d$-dimensional multivariate normal distribution $\eta \sim \cN (0, A\Lambda^{-1})$ where $A$ is a scalar. For any $\delta > 0$, with probability $1-\delta$,

\begin{equation*}
    \|\eta\|_{\Lambda} \leq c\sqrt{dA\log(d/\delta)},
\end{equation*}
where $c$ is some absolute constant. For $d=1$, we have $c=\sqrt{2}$.

\end{lemma}

\begin{lemma}\label{lemma:max_noise_union_bound}

Consider a  $d$-dimensional multivariate normal distribution $  \cN (0, A\Lambda^{-1})$ where $A$ is a scalar. Let $\eta_1, \eta_2, \ldots,\eta_M$ be $M$ independent samples from the distribution. Then for any $\delta > 0$

\begin{equation*}
\mathbb{P}\left(\max_{j \in [M]}\|\eta_{j}\|_{\Lambda} \leq c\sqrt{dA\log(d/\delta)}\right) \geq 1- M\delta, 
\end{equation*}
where $c$ is some absolute constant. 

\end{lemma}
\begin{proof}
From Lemma \ref{lemma-gaussian-concentration}, for a fixed $j \in [M]$, with probability at least $1-\delta$ we would have 
\begin{equation*}
    \|\eta\|_{\Lambda} \leq c\sqrt{dA\log(d/\delta)}.
\end{equation*}

Applying union bound over all $M$ samples completes the proof.
\end{proof}

\subsection{Inequalities for summations}
\begin{lemma}[Lemma D.1 in ~\cite{jin2019provably}] \label{Lemma:D.1-chijin}

Let $\Lambda_h = \lambda I + \sum_{i=1}^t \phi_i\phi_i^\top$, where $\phi_i \in \mathbb{R}^d$ and $\lambda > 0$. Then it holds that
\begin{equation*}
    \sum_{i=1}^t \phi_i^\top(\Lambda_h)^{-1}\phi_i \leq d.
\end{equation*}
\end{lemma}

\begin{lemma}[Lemma 11 in ~\cite{abbasi2011improved}]\label{Lemma:sum_of_feature_yasin}

Using the same notation as defined in this paper
\begin{equation*}
    \sum_{k=1}^K \bigl\|\phi(s_h^k,a_h^k)\bigr\|^2_{(\Lambda^k_h)^{-1}} \leq 2d\log\Bigl(\frac{\lambda + K}{\lambda}\Bigr).
\end{equation*}

\end{lemma}

\begin{lemma}\label{lemma:linear-algebra-norm-ineq}
Let $A \in \mathbb{R}^{d\times d}$ be a positive definite matrix where its largest eigenvalue $\lambda_{max}(A) \leq \lambda$. Let $x_1, \ldots, x_{k}$ be $k$ vectors in $\mathbb{R}^d$.  Then it holds that
\begin{equation*}
   \Bigl \|A\sum_{i=1}^{k}x_i \Bigr \|  \leq \sqrt{\lambda k}\Bigl(\sum_{i=1}^{k}\|x_i\|^2_{A}\Bigr)^{1/2}.
\end{equation*}
\end{lemma}
\begin{proof}
For any vector $v \in \mathbb{R}^d$, 
\begin{equation*}
    \begin{split}
        \|Av\| &= \|A^{1/2} A^{1/2} v\| \\ 
        &\leq \|A^{1/2}\| \|A^{1/2} v\| \\
        &= \|A^{1/2}\| \|v\|_A.
    \end{split}
\end{equation*}
Here the inequality follows from the definition of the operator norm $\|A^{1/2}\|$. Moreover, $\|A^{1/2}\| \leq \sqrt{\lambda}$ since $\lambda_{max}(A) \leq \lambda$. Thus,
\begin{equation} \label{eq:lambda}
    \left \|A\sum_{i=1}^{k}x_i\right \|  \leq \sqrt{\lambda} \left \|\sum_{i=1}^{k}x_i \right \|_A.
\end{equation}

Now by Cauchy-Schwarz inequality,
\begin{align}\label{Eq:Cauchy}
        \left \|\sum_{i=1}^{k} x_i \right \|_A^2 &= \sum_{i=1}^{k} \sum_{j=1}^{k} x_i^\top A x_j \notag\\
        &\leq \sum_{i=1}^{k}\sum_{j=1}^{k} \|x_i\|_A \|x_j\|_A \notag\\
        &= \left(\sum_{i=1}^{k} \|x_i\|_A\right)^2 \notag\\
        &\leq k \sum_{i=1}^{k} \|x_i\|_A^2.
\end{align}

Combining \eqref{eq:lambda} and \eqref{Eq:Cauchy}, proves the lemma.
\end{proof}

\subsection{Covering numbers and self-normalized processes}

\begin{lemma}[Lemma D.4 in \cite{jin2019provably}]\label{lemma:D.4-chi-jin}
Let $\{s_i\}_{i=1}^\infty$ be a stochastic process on state space $\cS$ with corresponding filtration $\{\cF_i\}_{i=1}^\infty$. Let $\{\phi_i\}_{i=1}^\infty$ be an $\mathbb{R}^d$-valued stochastic process where $\phi_i \in \cF_{i-1}$, and $\|\phi_i\| \leq 1$. Let $\Lambda_k = \lambda I + \sum_{i=1}^k \phi_i \phi_i^\top$. Then for any $\delta > 0$, with probability at least $1-\delta$, for all $k \geq 0$, and any $V \in \cV$ with $\sup_{s \in \cS} |V(s)| \leq H$, we have

\begin{equation*}
    \Bigl\|\sum_{i=1}^k \phi_i \bigl \{ V(s_i) - \mathbb{E}[V(s_i) \given \cF_{i-1} ]\bigr \}  \Bigr\|^2_{\Lambda_k^{-1}} \leq 4H^2 \Bigl [\frac{d}{2}\log\Bigl(\frac{k+\lambda}{\lambda}\Bigr) + \log{\frac{\cN_\varepsilon}{\delta}}\Bigr] + \frac{8k^2\epsilon^2}{\lambda},
\end{equation*}
where $\cN_\varepsilon$ is the $\varepsilon$-covering number of $\cV$ with respect to the distance $\text{dist}(V,V') = \sup_{s \in \cS} |V(s) - V'(s)|$.
\end{lemma}

\begin{lemma}[Covering number of Euclidean ball, \cite{vershynin2018high} ]\label{lemma:covering-number-euclidean-ball}
For any $\varepsilon > 0$, the $\varepsilon$-covering number, $\cN_\varepsilon$, of the Euclidean ball of radius $B > 0$ in $\mathbb{R}^d$ satisfies 
\begin{equation*}
    \cN_\varepsilon \leq \Bigl(1+ \frac{2B}{\varepsilon}\Bigr)^d \leq \Bigl( \frac{3B}{\varepsilon}\Bigr)^d .
\end{equation*}
\end{lemma}

\begin{lemma}\label{lemma:covering_number_of_V(.)}
Consider a class of functions $\cV: \cS \rightarrow \RR$ which has the following parametric form

\begin{equation*}\label{Eq:parametric-function-class}
    V(\cdot) = \Bigl\la \min \bigl\{   \phi(\cdot,\cdot)^\top\theta, H\bigr\}^{+}, \pi(\cdot \given \cdot) \Bigr \ra_{\cA},
\end{equation*}
where the parameter $\theta$ satisfies $\|\theta\| \leq B$ and for all $(s,a) \in \cS\times\cA$, we have $\|\phi(s,a)\| \leq 1$. If $\cN_{\cV,\varepsilon}$ denotes the $\varepsilon$-covering number of $\cV$ with respect to the distance $\text{dist}(V,V') = \sup_{s\in \cS} |V(s) - V'(s)|$, then
\begin{equation*}
    \log \cN_{\cV,\varepsilon} \leq d\log(1+ 2B/\varepsilon) \leq d\log(3B/\varepsilon).
\end{equation*}
\end{lemma}

\begin{proof}
Consider any two functions $V_1, V_2 \in \cV$ with parameters $\theta_1$ and $\theta_2$, respectively. Note that  $\min\{\cdot, H\}$ is a contraction mapping. Thus we have

\begin{align}\label{eq:dist(v_1,v_2)}
        \text{dist}(V_1, V_2) &\leq \sup_{s} \bigl| \la \phi(s,\cdot)^\top\theta_1 - \phi(s,\cdot)^\top\theta_2, \pi(\cdot \mid s) \ra_{\cA} \bigr| \notag\\
        & \leq \sup_{\phi:\|\phi\|\leq 1} \bigl|\phi^\top\theta_1 - \phi^\top\theta_2 \bigr| \notag\\
        & = \sup_{\phi:\|\phi\|\leq 1} \bigl|\phi^\top\bigl(\theta_1 - \theta_2\bigr) \bigr| \notag\\
        & \leq \sup_{\phi:\|\phi\|\leq 1} \|\theta_1- \theta_2 \|_2 \|\phi \|_2 \notag\\
        &= \|\theta_1 - \theta_2\|,
\end{align}

where the second inequality follows from the triangle inequality and the third inequality follows from the Cauchy-Schwarz inequality. 

If $\cN_{\theta,\varepsilon}$ denotes the $\varepsilon$-covering number of $\{\theta \in \mathbb{R}^d \given \|\theta\| \leq B \}$, Lemma ~\ref{lemma:covering-number-euclidean-ball} implies
\begin{equation*}
    \cN_{\theta,\varepsilon} \leq \Bigl(1+ \frac{2B}{\varepsilon}\Bigr)^d \leq \Bigl( \frac{3B}{\varepsilon}\Bigr)^d.
\end{equation*}

Let $\cC_{\theta,\varepsilon}$ be an $\varepsilon$-cover of $\{\theta \in \mathbb{R}^d \given \|\theta\| \leq B \}$ with cardinality $\cN_{\theta,\varepsilon}$. Consider any $V_1 \in \cV$. By  \eqref{eq:dist(v_1,v_2)}, there exists $\theta_2 \in \cC_{\theta,\varepsilon}$ such that $V_2$ parameterized by $\theta_2$ satisfies $\text{dist}(V_1,V_2)\leq \varepsilon$. Thus we have

\begin{equation*}
    \log\cN_{\cV,\varepsilon} \leq \log\cN_{\theta,\varepsilon} \leq d\log(1+ 2B/\varepsilon) \leq d\log(3B/\varepsilon),
\end{equation*}
which concludes the proof. 
\end{proof}
\section{Experiment Details}
In this section we include the figure for the RiverSwim environment from \citep{osband2013more}.
\begin{figure}[h]
    \centering
    \includegraphics[width=0.9\textwidth]{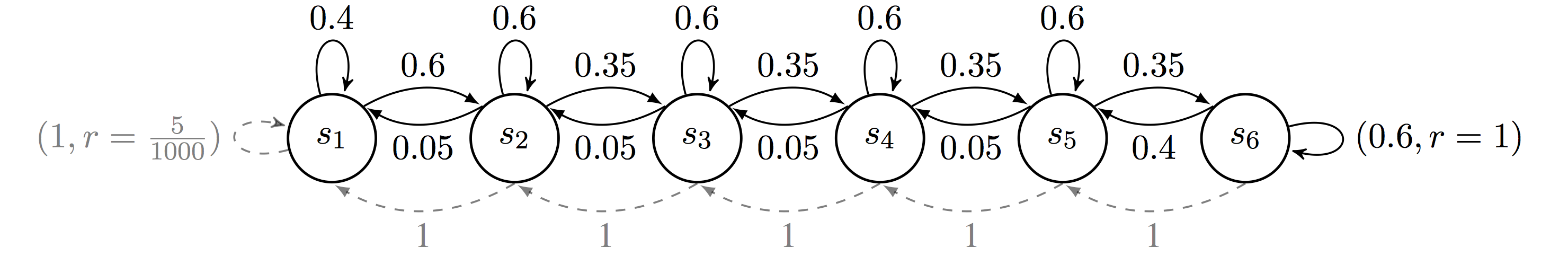}
    \caption{The 6 state RiverSwim environment \citep{osband2013more}. State $s_1$ has a small reward while state $s_6$ has a large reward. The action whose transition is denoted with a dashed arrow deterministically moves the agent left. The other action is stochastic, and with relative high probability moves the agent towards the goal state $s_6$. This action represents swimming against the current, hence the name RiverSwim.}
    \label{fig:riverswim}
\end{figure}

\newpage
\end{document}